%% file: main.tex
\documentclass{article}
\usepackage{iclr2027_conference,times}
\usepackage{amsmath,amssymb,amsthm,mathtools}
\usepackage{microtype,graphicx,booktabs,tabularx,enumitem,needspace}
\usepackage{placeins}
\usepackage{tikz}
\usetikzlibrary{arrows.meta,positioning}
\usepackage{hyperref,url}
\usepackage[nameinlink,capitalise,noabbrev]{cleveref}
\newtheorem{theorem}{Theorem}
\newtheorem{lemma}[theorem]{Lemma}
\newtheorem{proposition}[theorem]{Proposition}

\theoremstyle{definition}
\newtheorem{definition}[theorem]{Definition}
\theoremstyle{remark}

\newcommand{\R}{\mathbb{R}}
\newcommand{\N}{\mathbb{N}}
\newcommand{\supp}{\operatorname{supp}}
\newcommand{\spann}{\operatorname{span}}
\newcommand{\conv}{\operatorname{conv}}
\newcommand{\relint}{\operatorname{relint}}
\newcommand{\diag}{\operatorname{diag}}
\newcommand{\dist}{\operatorname{dist}}
\newcommand{\argmin}{\operatorname*{arg\,min}}
\newcommand{\wstar}{w^\star}
\newcommand{\LD}{L_D^\star}
\newcommand{\LDS}[1]{L_D^\star(#1;A^\star)}
\newcommand{\GammaDN}[1]{\Gamma_d(#1)}

\title{Weighted Data Selection:\\Sharp Upper-Half and Five-Dimensional Laws}
\author{Zhongxuan Liu \qquad Hongzhi Wang\thanks{Corresponding author.}\\
Faculty of Computing, Harbin Institute of Technology\\
\texttt{lazrix@163.com} \quad \texttt{wangzh@hit.edu.cn}}
\iclrfinalcopy
\hypersetup{pdftitle={Weighted Data Selection: Sharp Upper-Half and Five-Dimensional Laws},pdfauthor={Zhongxuan Liu, Hongzhi Wang},pdfsubject={}}
\begin{document}
\maketitle
\lhead{Preprint}
\begin{abstract}
How much risk does a small reweighted training support retain? For finite
weighted least squares with the minimum-norm learner, we prove the exact law
$\Gamma_d(n)=3-n/d$ throughout $\lceil3d/2\rceil\leq n\leq2d-1$.
The guarantee covers every observed feature rank and uses selections that
preserve the full feature span. Balanced simplex anchors reduce dimension;
positive-weight lifting and independent-line compression close the risk
bound. Shifted coordinate pairs attain the matching lower bound. The
complete dataset-level upper bound and sharpness construction are verified
in Lean~4. At the smaller budget $(d,n)=(5,6)$, we also prove
$\Gamma_5(6)=11/5$, matching the simplex-block prediction from $5=3+2$
over arbitrary interacting configurations. Circuit covers, comparison
second moments, and circuit-plane probabilities give the sharp excess
$6/5$, while polar-face geometry resolves shared rank-three circuits.
The general simplex-block frontier connects these laws within the
intermediate-budget selection problem.
\end{abstract}
\input{sections/main_text}
\subsection*{Reproducibility statement}
Appendices~\ref{app:proofs} and~\ref{app:five-six} contain complete proofs,
including zero residuals, deficient ranks, and positive-weight limits.
Appendix~\ref{app:verification} maps the supplied symbolic checker to the
algebra it tests and gives reproduction commands. The checker verifies the stated finite algebraic identities.
Appendix~\ref{app:upper-half} gives the upper-half proof and its Lean~4
theorem map; the supplied formal project checks the full statement and
its foundational dependencies. Appendix~\ref{app:landscape} proves the block
frontier and gives the matrix formulation for further intermediate budgets.
\label{end:main}
\bibliographystyle{iclr2027_conference}
\bibliography{references}
\clearpage
\appendix
\numberwithin{equation}{section}
\section{Full certificate statements and boundary proofs}
\label{app:proofs}
\subsection{Exact certificates and intrinsic-rank rigidity}
\label{sec:geometry}
\input{appendix/geometry}
\subsection{Exact one-dimensional compression}
\label{sec:one-dimensional}
\input{appendix/one_dimensional}
\subsection{Instance-dependent and sharp boundary laws}
\label{sec:boundary}
\input{appendix/sharp_boundary}
\input{appendix/proofs}
\clearpage
\section{Complete proof of the five-dimensional six-point law}
\label{app:five-six}
\input{appendix/five_dimensional_proof}
\section{Complete proof of the upper-half law}
\label{app:upper-half}
\input{appendix/upper_half_verified}
\section{The block frontier and the matrix formulation}
\label{app:landscape}
\input{appendix/landscape}
\section{Exact verification and proof index}
\label{app:verification}
\input{appendix/verification}
\end{document}

%% file: sections/main_text.tex
\section{Introduction}
\label{sec:introduction}
Data selection changes the training data while keeping the learner fixed.
For least squares, a selector can retain a few original examples and assign
them new positive weights. The question is how well the resulting predictor
performs on the full dataset. This separates the information carried by a
small training support from the computational method used to find it.
In the very small-support regime, even the rule for choosing among multiple
least-squares solutions matters: balancing the selected gradients need
not recover the full solution if the selected feature span loses rank.

We study exact worst-case risk inflation for the minimum-norm learner.
The evaluation weights $p$ specify the target finite distribution; they
may encode multiplicities or importance weights. The training weights
$q$ are the selector's decision variables. This distinction lets a small
support represent a nonuniform target without changing the evaluation
criterion. The guarantee concerns this finite-distribution risk, rather
than population generalization or the running time of a selection algorithm.

We establish two sharp laws for arbitrary weighted configurations.
Throughout the upper half of the intermediate-budget interval,
$\Gamma_d(n)=3-n/d$ for $\lceil3d/2\rceil\leq n\leq2d-1$.
Balanced anchors and full-span lifting give a dimension induction whose
complete dataset-level statement is verified in Lean~4. At the smaller
budget $(d,n)=(5,6)$, the unrestricted bound $\Gamma_5(6)\leq11/5$
matches a seven-point example. This second result handles intersecting
positive circuits through more detailed five-dimensional geometry.

\begin{figure}[t]
\centering
\includegraphics[width=\textwidth]{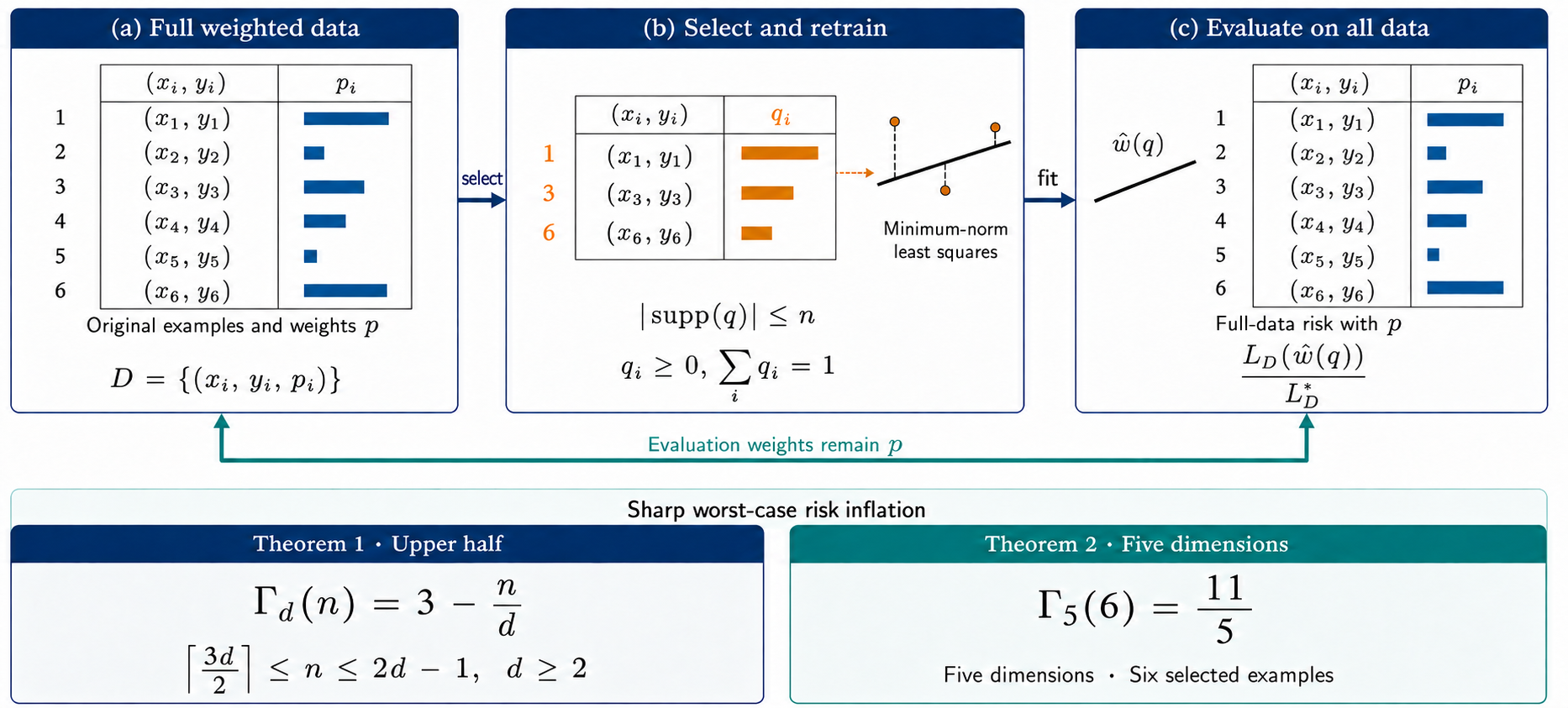}
\caption{Weighted selection with a fixed learner. A support of at most $n$
original examples is reweighted for minimum-norm least-squares training;
evaluation uses the complete dataset and its original weights $p$.
The two sharp laws cover the upper-half budget interval and the
five-dimensional six-example problem, respectively. Tables and fitted
lines are schematic; selected rows retain their original features and responses.}
\label{fig:selection-problem}
\end{figure}

\paragraph{Prior results and the precise addition.}
\citet{hanneke2025data} established the uniform-evaluation endpoints
$\Gamma_d(d)=d+1$ and exact recovery at $2d$; the companion open problem
asked about intermediate budgets \citep{hanneke2025openproblem}.
Table~\ref{tab:positioning} separates the results used here from the
five-dimensional upper bound. In particular, the value $11/5$ as a lower bound is already
present in \citet{zhang2026exact}, version~1, via its simplex-block
construction. That work also proves the boundary law and the four-dimensional
five-point law. We use its latter theorem as an input. Our contribution
at $(5,6)$ is the universal matching upper bound, together with a
self-contained weighted, minimum-norm-safe matching example.

\begin{table}[t]
\centering
\caption{Result-level positioning relative to \citet{zhang2026exact},
version~1. Its original evaluation model is uniform. The weighted
transfer of the four-dimensional input is proved in
Appendix~\ref{app:five-transfer}.}
\label{tab:positioning}
\small
\begin{tabularx}{\linewidth}{@{}lXX@{}}
\toprule
Result & Zhang (2026), v1 & Role in this paper \\
\midrule
$\Gamma_d(2d-1)=1+1/d$ & Proved & Weighted certificate proof and rigidity \\
$\Gamma_4(5)=2$ & Proved (Theorem 47) & Lower-dimensional input \\
$\Gamma_5(6)\geq11/5$ & Known lower bound & Explicit attained weighted example \\
$1+\Phi_{d,s}$ & Block bound and structured frontier & Weighted formulation and connection to $(5,6)$ \\
$\Gamma_5(6)\leq11/5$ & Left open & Universal upper bound \\
\bottomrule
\end{tabularx}
\end{table}

\Needspace{4\baselineskip}
\paragraph{A training-realization viewpoint.}
The proof asks which low-risk geometric points can actually be returned
by the stipulated learner. A positive circuit of rank $r$ has $r+1$
rows whose positive weighted sum is zero. Combining it with $5-r$
independent interpolating rows gives a six-row full-rank training
problem. This interface allows polar geometry, comparison second moments,
and circuit-plane graphs to control the risk of actual fits.
The main text derives both sources of the sharp excess $6/5$ and states
the exhaustive classification connecting them. The appendix retains the
full shared-circuit analysis rather than assuming that all instances split
into independent blocks.

The same concern with realizability motivates our complementary
certificate theory. Full-span certificate complexity identifies when
exact recovery is available and when maximal complexity forces independent
one-dimensional blocks. Exact one-point penalties on those blocks yield
the weighted boundary law and its extremal conditions. Thus the boundary
and interior results share an interface between positive dependence and
the output of a fixed learner, while requiring different geometry.

\paragraph{Relation to approximation and sampling.}
Regression coresets approximate a full objective or its solution
\citep{feldman2011unified,drineas2006sampling,dasgupta2009sampling,
boutsidis2013nearoptimal}. Sketching and leverage-based methods provide
computational and statistical guarantees
\citep{clarkson2009streaming,clarkson2017input,pilanci2015randomized,
ma2015leveraging,raskutti2016sketching}; volume sampling gives randomized
subset guarantees \citep{derezinski2018volume}. Here we instead determine
a deterministic infimum over reweighted original supports at a fixed
small budget, including deficient-rank designs. Positive-basis theory
\citep{davis1954positive,naevdal2019positive} supplies geometric inputs;
the additional requirement is to implement the resulting certificate
by positive-weight training with the correct minimum-norm output.

\section{Problem and sharp intermediate-budget laws}
\label{sec:formulation}
Let $D=\{(x_i,y_i,p_i)\}_{i=1}^N$, with $x_i\in\R^d$, $y_i\in\R$,
$p_i>0$, and $\sum_i p_i=1$. Write $[N]=\{1,\ldots,N\}$ and define
\[
 L_D(w)=\sum_i p_i(x_i^\top w-y_i)^2,\qquad
 \wstar=\argmin_{\text{minimum norm},\,w}L_D(w),\qquad \LD=L_D(\wstar).
\]
Minimum norm is the generalized-inverse least-squares convention
\citep{penrose1955inverse,penrose1956best}. Set
$U=\spann\{x_i\}$; then $\wstar\in U$.
For $q\in\Delta_N=\{q\geq0:\sum_iq_i=1\}$, let $\widehat w(q)$ be the
minimum-norm minimizer of $\sum_iq_i(x_i^\top w-y_i)^2$. Only positive
support entries count toward the budget. Define
\begin{equation}
 \LDS{n}=\inf_{\substack{q\in\Delta_N\\|\supp(q)|\leq n}}
 L_D(\widehat w(q)),
 \label{eq:selected-risk}
\end{equation}
\begin{equation}
 \Gamma_d(n)=\sup_{D:\LD>0}\frac{\LDS{n}}{\LD}.
 \label{eq:gamma}
\end{equation}
The evaluation weights stay equal to $p$, even though training uses $q$.
We keep an infimum because minimum-norm solutions can jump at a rank
change. When $\LD=0$, selecting a basis of $U$ with arbitrary positive
weights recovers $\wstar$: the selected solutions are
$\wstar+U^\perp$, whose minimum-norm point is $\wstar$. Thus zero optimal
risk is recovered at budgets $n\geq d$; the zero-rank case is immediate.

\begin{theorem}[Upper-half intermediate-budget law]
\label{thm:upper-half}
For $d\geq2$ and $1\leq t\leq\lfloor d/2\rfloor$,
\[
 \Gamma_d(2d-t)=1+\frac td,
 \qquad\text{equivalently}\qquad
 \Gamma_d(n)=3-\frac nd\quad(\lceil3d/2\rceil\leq n\leq2d-1).
\]
For every finite weighted instance with $\LD>0$ and every $\varepsilon>0$,
there is $q\in\Delta_N$ with $|\supp(q)|\leq2d-t$ and $U_q=U$ such that
$L_D(\widehat w(q))\leq(1+t/d+\varepsilon)\LD$.
Uniform pairs $(e_j,1),(e_j,3)$ attain the matching lower bound.
\end{theorem}
Appendix~\ref{app:upper-half} proves the theorem by dimension induction.
A balanced $a$-dimensional simplex fixes an anchor, and a full-rank
selection in its orthogonal complement lifts to the original problem.
The budget and risk ratios close when $d\geq2t$. Maximal certificates
reduce to independent lines, where compressing the $t$ least costly lines
gives excess at most $(t/d)\LD$. The Lean~4 proof includes zero residuals,
deficient ranks, normalized weights, and the stipulated minimum-norm rule.

\Needspace{8\baselineskip}
\begin{theorem}[Five-dimensional six-point law]
\label{thm:five-six}
For the weighted minimum-norm selection problem in \cref{eq:gamma},
\[
 \Gamma_5(6)=\frac{11}{5}.
\]
More precisely, for every finite instance in $\R^5$ and every
$\varepsilon>0$, a selection of at most six points, with strictly positive
weights on its support, satisfies
$L_D(\widehat w)\leq(11/5+\varepsilon)L_D^\star$.
For positive observed feature rank, the selection can span that feature
space. A seven-point instance attains the ratio $11/5$.
\end{theorem}

The upper bound makes no separation assumption on the features. Full span
is essential to the proof: every approximating training objective has a
unique minimizer on $U$, so it returns the required minimum-norm vector.
The statement allows a limiting sequence; it does not assert that every
instance attains its selected infimum. By contrast, the lower-bound
example below attains the exact ratio with a finite selection.

\subsection{Intermediate-budget landscape}
\label{sec:landscape}
Write $n=d+s$, where $0\leq s\leq d-1$. A block of dimension $a$
can be represented exactly by its $a+1$ circuit rows. Spending only $a$
rows incurs excess at most $a$ times its optimal block loss. The budget
$d+s$ therefore allows $s$ blocks to receive the extra row.
This allocation problem leads to
\begin{equation}
 \Phi_{d,s}=\max_{\substack{s<k\leq d\\a_1+\cdots+a_k=d\\a_j\in\N_{>0}}}
 \frac{k-s}{\sum_{j=1}^k a_j^{-1}}.
 \label{eq:landscape-phi}
\end{equation}
Let $\Gamma_d^\oplus(n)$ restrict \eqref{eq:gamma} to orthogonal sums of
simplex circuits: block $j$ consists of $a_j+1$ feature rows, spans an
$a_j$-dimensional subspace, and every $a_j$ rows form a basis; its
nonzero residual-gradient atoms form a positive circuit at the block optimum.
Evaluation weights and labels may vary within this class.
\begin{proposition}[General lower bound and exact block frontier]
\label{prop:landscape}
For $d\geq1$ and $0\leq s\leq d-1$,
\begin{equation}
 \Gamma_d(d+s)\geq1+\Phi_{d,s},\qquad
 \Gamma_d^\oplus(d+s)=1+\Phi_{d,s}.
 \label{eq:landscape-frontier}
\end{equation}
\end{proposition}
Appendix~\ref{app:landscape} gives the block allocation proof and a translated
construction that also excludes deficient-rank selections. These recover
the block lower bound and structured frontier of
\citet[Proposition~12, Theorem~15, Corollary~16]{zhang2026exact}
in the present weighted model. Their role here is to identify the extremal
prediction tested by the new unrestricted upper bound.

For $s=1$, $\Phi_{d,1}=\lfloor d^2/4\rfloor/d$ for $d\geq2$.
In particular, at $d=5$ the partition $3+2$ gives
\begin{equation}
 c=(1/3+1/2)^{-1}=6/5,\qquad
 (\ell_3,\ell_2)=(c/3,c/2)=(2/5,3/5).
 \label{eq:landscape-five}
\end{equation}
The extra row completes one block, while the other contributes $6/5$.
Theorem~\ref{thm:five-six} proves that arbitrary interacting circuits obey
the same worst-case value. The three-plus-two upper-bound cover below
recovers precisely this balancing mechanism.

The frontier also gives $\Phi_{d,0}=d$ and $\Phi_{d,d-1}=1/d$.
These match the endpoint and the penultimate boundary, respectively;
full-span certificates give exact recovery at $n\geq2d$.
Theorem~\ref{thm:upper-half} settles the whole upper-half interval;
Theorem~\ref{thm:five-six} settles $(5,6)$ below it. The general equality
$\Gamma_d(d+s)=1+\Phi_{d,s}$ remains the target for the other intermediate
budgets. Appendix~\ref{app:matrix-target} gives its matrix formulation.

\subsection{Proof architecture}
Figure~\ref{fig:proof-routes} organizes the two sharp laws into separate
proof routes. The upper route combines full-span certificates, balanced
anchors, and independent-line compression through dimension induction.
The lower route classifies normalized five-dimensional systems and
transfers realizable fits back to regression. Each universal upper bound
is paired with its matching lower-bound construction.

\begin{figure}[t]
\centering
\includegraphics[width=\textwidth]{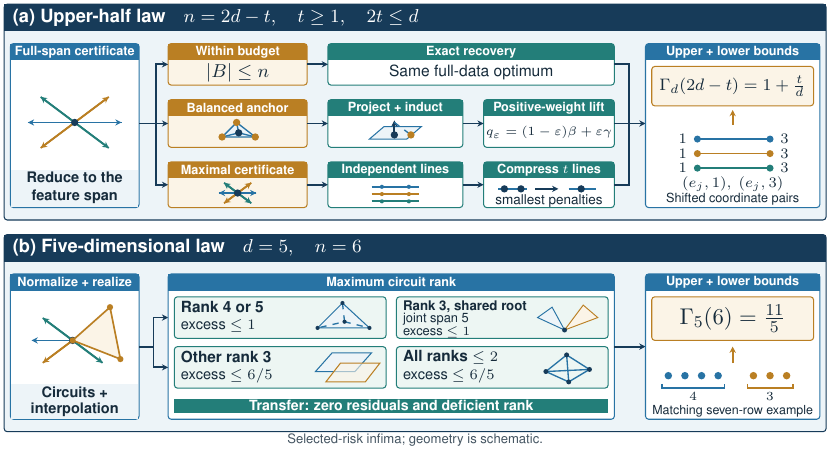}
\caption{Proof routes for the two main theorems. Top: full-span
certificate analysis gives exact recovery within budget, a balanced anchor
for projection and positive-weight lifting, or independent lines whose
$t$ cheapest coordinates can be compressed. Shifted coordinate pairs
give the matching lower bound. Bottom: four exhaustive circuit-rank
configurations bound the five-dimensional normalized excess by $6/5$;
transfer and a seven-row example yield $11/5$. The bounds concern
selected-risk infima; $B$ denotes a certificate support and the geometry is schematic.}
\label{fig:proof-routes}
\label{fig:case-map}
\label{fig:certificate-boundary}
\end{figure}
\FloatBarrier

\subsection{Normalization and the training interface}
\label{sec:realization}
For full observed rank and $\LD>0$, let $H=\sum_i p_ix_ix_i^\top$ and set
\[
 \xi_i=H^{-1/2}x_i,\quad
 \eta_i=(y_i-x_i^\top\wstar)/\sqrt{\LD},\quad
 z=H^{1/2}(w-\wstar)/\sqrt{\LD}.
\]
The normal equations give $\sum_i p_i\eta_i\xi_i=0$ and
$\sum_i p_i\eta_i^2=1$, while $L_D(w)/\LD=1+\|z\|^2$.
In the case $\eta_i\ne0$, $\xi_i\ne0$ for every row, the roots
$v_i=\xi_i/\eta_i$ and probabilities $\lambda_i=p_i\eta_i^2$ satisfy
\begin{equation}
 \lambda_i>0,\quad \sum_i\lambda_i=1,\quad
 \sum_i\lambda_iv_i=0,\quad
 \sum_i\lambda_iv_iv_i^\top=I_5.
 \label{eq:main-normal}
\end{equation}
The normalized training responses are all one. A root weight $w_i$
corresponds to original training weight proportional to $w_i/\eta_i^2$.
Write $\delta_6(V)$ for the infimum of $\|z\|^2$ over positive-weight
full-rank fits using at most six roots. Proving $\delta_6(V)\leq6/5$
gives the desired ratio in this normalized case.

A \emph{positive circuit} is a support-minimal positive zero relation
among roots. A circuit $A$ of rank $r$ has $r+1$ roots. Every root belongs
to a circuit, since the strictly positive zero relation in
\eqref{eq:main-normal} decomposes into extreme rays of its nonnegative
relation cone.

\begin{lemma}[From a circuit to an actual fit]
\label{lem:main-realization}
Let $A$ be a rank-$r$ positive circuit with span $P$, and let $J$ be
$5-r$ roots independent modulo $P$. If $v_j^\top z=1$ for $j\in J$
and $v_i^\top z\leq1$ for $i\in A$, then $z$ is a limit of
positive-weight full-rank training outputs on $A\cup J$.
\end{lemma}
\begin{proof}
Take $c_i>0$ with $\sum_{i\in A}c_iv_i=0$, and let $z_0\in P^\perp$
interpolate $J$. Then $z_\epsilon=(1-\epsilon)z+\epsilon z_0$ still
interpolates $J$ and has $v_i^\top z_\epsilon<1$ on $A$. Set
\[
 w_i=\frac{c_i}{1-v_i^\top z_\epsilon}\quad(i\in A)
\]
and give the interpolated rows arbitrary positive weights. The gradient
is $-2\sum_{i\in A}c_iv_i=0$. Since the selected rows span $\R^5$,
this stationary point is the unique minimizer. The support size is
$(r+1)+(5-r)=6$. Let $\epsilon\downarrow0$.
\end{proof}

This argument is dimension-independent with $5$ replaced by $d$ and
support budget $d+1$. The remaining challenge is to find a low-norm point
meeting its interpolation and one-sided circuit inequalities.

\section{The exhaustive geometry behind the upper bound}
\label{sec:cases}
Call a root \emph{good} if it is a strictly positive combination of five
other independent roots; call the remaining roots \emph{exceptional},
and denote their set by $E$. Put $K=\conv\{\pm v_i\}$.
A good root on $\partial K$ yields $\delta_6\leq1$ by a supporting
functional and an explicit six-row fit. Otherwise the exceptional roots
generate $K$ and admit a comparison second moment
\begin{equation}
 K=\conv\{\pm e:e\in E\},\qquad
 C=\sum_{e\in E}\mu_e ee^\top\succeq I_5,
 \qquad \mu_e\geq0,\quad\sum_e\mu_e=1.
 \label{eq:main-comparison}
\end{equation}
Indeed, express each original root as a convex combination of signed
exceptional roots and average the quadratic Jensen inequalities with
$\lambda_i$. The measure $\mu$ need not have mean zero. This distinction
is important when it is used for interpolation below; it is a second
moment comparison, not a centered-covariance replacement.
Appendix~\ref{app:five-interfaces} supplies the boundary-root construction.

\begin{proposition}[Exhaustive circuit-rank split]
\label{prop:exhaustive}
For every finite root system satisfying \eqref{eq:main-normal}, exactly
one of the following circuit configurations occurs.
\begin{enumerate}[leftmargin=*,label=(\roman*)]
\item A positive circuit has rank four or five. Then $\delta_6\leq1$.
\item All circuits have rank at most three, and two rank-three circuits
share a root and jointly span $\R^5$. Then $\delta_6\leq1$.
\item A rank-three circuit exists, every circuit has rank at most three,
and the shared-root configuration in (ii) is absent. Then either the
two-line quotient argument gives $\delta_6\leq1$, or two circuit spans
of dimensions three and two cover $E$, giving $\delta_6\leq6/5$.
\item All circuits have rank at most two. Either a four-plus-one
decomposition (or an independent-line system) gives $\delta_6\leq1$,
or the circuit-plane graph gives $\delta_6\leq6/5$.
\end{enumerate}
In (ii)--(iii), the good-boundary-root alternative is first discharged
with bound one; the exceptional-root assertions concern the other alternative.
\end{proposition}

\paragraph{Why the split is sufficient.}
The four configurations are disjoint and exhaustive by the maximum
circuit rank and the presence of a shared-root rank-three pair. A rank-five
circuit trains exactly at zero. For a rank-four circuit, one maximum
absolute normal projection supplies the interpolating row and squared
norm at most one. The shared-root case is handled by signed polar faces,
as explained below and proved in Appendix~\ref{app:five-shared}.
For (iii), quotienting by a rank-three circuit either yields only two
independent lines, or yields a seven-root positive generating subset.
The two extreme rays of its relation cone have supports of sizes
$(4,3)$, disjoint, or $(4,4)$, sharing one root. The latter is excluded
by (iii). In the former, any crossing exceptional root either forces
the forbidden shared-root pair or reduces all exceptional projections
to two quotient lines; with no crossing exceptional roots the $(3,2)$
cover remains. The support and crossing arguments are given in
Appendix~\ref{app:five-remaining}. Finally, circuit-plane intersections
give the alternatives in (iv), proved in Appendix~\ref{app:five-lowrank}.
The next two subsections derive the constant in the branches that need
the full excess $6/5$.

\subsection{First mechanism: a three-plus-two cover}
\label{sec:cover}
For any rank-$r$ anchor circuit with span $P$, the comparison measure gives
\begin{equation}
 \delta_6\leq(5-r)\mu(E\setminus P).
 \label{eq:main-anchor}
\end{equation}
To see why this remains valid for a noncentered measure, project its
roots onto $P^\perp$ and write their second moment as
$B=\sum_i\eta_i z_iz_i^\top\succ0$, total mass as $a=\sum_i\eta_i$,
and first moment as $b=\sum_i\eta_i z_i$. Determinant-weighted sampling
of $k=5-r$ independent rows yields interpolants $t_J$ satisfying
\begin{equation}
 \mathbb E[t_J^\top Bt_J]
 \leq ka-(k-1)b^\top B^{-1}b\leq ka.
 \label{eq:main-interpolation}
\end{equation}
This follows by comparing the coefficient of a common rank-one row
translation in the Cauchy--Binet identity; singular subsets contribute
nonnegative terms (Lemma~\ref{lem:five-interpolation}). Since $B$ dominates
the identity on $P^\perp$, some interpolant has squared norm at most
$ka$. Its anchor evaluations are zero, so
Lemma~\ref{lem:main-realization} gives an actual admissible fit.
Only roots outside $P$ enter $a$, proving \eqref{eq:main-anchor}.

Now let circuit spans $P_A,P_B$ of dimensions three and two cover $E$.
Write $\alpha=\mu(E\setminus P_A)$ and
$\beta=\mu(E\setminus P_B)$. The two outside sets are disjoint, so
$\alpha+\beta\leq1$. Apply \eqref{eq:main-anchor} to each anchor:
\begin{equation}
 \delta_6\leq\min\{2\alpha,3\beta\}
 \leq\min\{2\alpha,3(1-\alpha)\}\leq\frac65.
 \label{eq:main-cover}
\end{equation}
The last maximum occurs when $2\alpha=3(1-\alpha)$, that is,
$\alpha=3/5$. This explains the constant without assuming the comparison
measure is centered or replacing the original risk by a block surrogate.
The same calculation for a $(3,3)$ cover gives bound one.

\subsection{Second mechanism: circuit-plane graphs and actual risk}
\label{sec:graph}
Suppose all circuits have rank at most two. Form a graph $G$ whose
vertices are distinct planes of rank-two circuits, joining two planes
when their intersection is zero. If $G$ has a triangle or four-cycle,
the circuit geometry forces every root into a four-space and a
complementary line. Their evaluation second moments are orthogonal
projections, so the four-dimensional five-row input and one interpolated
line use six rows with total excess at most one.

Otherwise $G$ has no triangles or four-cycles. Neighbors of a vertex
are pairwise intersecting planes. They either share a line or lie in
a three-space. Circuit elimination bounds these alternatives by two
and three neighbors, respectively (Lemma~\ref{lem:five-planes}), hence
$\Delta(G)\leq3$. If there are no circuit planes, all roots lie on five
independent lines and the maximum-absolute-coefficient interpolant gives
bound one. We therefore consider a nonempty $G$.

\begin{lemma}[Graph probability]
\label{lem:main-graph}
Every nonempty finite graph with no triangles or four-cycles and
$\Delta(G)\leq3$ admits a probability $\theta$ on its vertices with
\begin{equation}
 \sum_{A\sim B}\theta_A-\theta_B\leq\frac15
 \quad\text{for every }B.
 \label{eq:main-graph}
\end{equation}
\end{lemma}
\begin{proof}
If a vertex has degree three, put mass $2/5$ on it and $1/5$ on each
neighbor. Every exterior vertex meets at most one of those neighbors,
since two would form a four-cycle. This proves the bound everywhere.
If the maximum degree is at most two, choose one path or cycle component.
Use the uniform probability on a cycle of length at least five.
For paths on one through four vertices use, respectively,
$(1)$, $(1/2,1/2)$, $(2/7,3/7,2/7)$, and
$(1/6,1/3,1/3,1/6)$. For longer paths put $1/5$ on five consecutive
vertices. All other vertices have zero mass; direct substitution proves
the claim, including vertices just outside the positive-mass segment.
\end{proof}

The constant $1/5$ is sharp for this abstract graph statement: on a
five-cycle, summing its five left sides gives one for every probability,
so their maximum is at least $1/5$, and uniform mass attains it.
This observation is not a claim that every eligible graph is realizable
by a normalized root system, nor a separate regression lower bound.

\paragraph{Why this graph controls training risk.}
Choose an actual triangle circuit as anchor for each graph vertex $A$.
The polytope
\[
 \mathcal P_A=\{z\in P_A^\perp:v_i^\top z\leq1\text{ for every }i\}
\]
is bounded and three-dimensional. Express zero as a convex combination
of its vertices. Each vertex interpolates three rows independent modulo
$P_A$; together with the zero-evaluation anchor it is an actual full-rank
six-row fit. This finite distribution has $\mathbb E_Az=0$.

For a circuit $B$ with normalized coefficients $\beta_i$, let
$\rho_{AB}=\dim P_B-\dim(P_A\cap P_B)$. Project the positive relation
of $B$ modulo $P_A$ and decompose it into projected circuits. On one
such circuit with coefficients $\alpha_i$ and support size $s$, write
$T_i=v_i^\top z$. We have $T_i\leq1$, $\sum_i\alpha_iT_i=0$
pointwise, and $\mathbb E_AT_i=0$. Consequently
\[
 -\frac{1-\alpha_i}{\alpha_i}\leq T_i\leq1,
 \qquad \mathbb E_AT_i^2\leq\frac{1-\alpha_i}{\alpha_i},
 \qquad \mathbb E_A\sum_i\alpha_iT_i^2\leq s-1\leq\rho_{AB}.
\]
The second inequality follows by averaging
$(T_i+(1-\alpha_i)/\alpha_i)(T_i-1)\leq0$.
Reassembling the projected circuits, whose total mass is at most one,
gives $\mathbb E_A\sum_{i\in B}\beta_i(v_i^\top z)^2\leq\rho_{AB}$.
For rank-two $B$,
\[
 \rho_{AB}=1+\boldsymbol1_{A\sim B}-\boldsymbol1_{P_A=P_B},
 \qquad \sum_A\theta_A\rho_{AB}\leq1+\frac15=\frac65;
\]
for rank-one $B$ the same average is at most one. Decompose the original
zero-mean probability $\lambda$ into normalized positive circuits and
mix the anchor distributions according to $\theta$. The \emph{original}
second moment in \eqref{eq:main-normal} now gives
\begin{equation}
 \mathbb E\|z\|^2
 =\mathbb E\sum_i\lambda_i(v_i^\top z)^2\leq\frac65.
 \label{eq:main-realrisk}
\end{equation}
Some fit in this finite distribution therefore has excess at most $6/5$.
Thus the graph bound controls the risk of realizable training outputs,
not merely a numerical quantity attached to circuit intersections.

\subsection{The shared-root branch: the invariant that closes the cases}
\label{sec:shared}
Two rank-three circuits sharing a root and spanning five dimensions
can, for incidence calculations, be positively rescaled to
\[
 u+a_0+a_1+a_2=0,\qquad u+b_0+b_1+b_2=0.
\]
Additional roots cannot simply be discarded. The proof classifies each
crossing exceptional root into three forms, then checks which forms
can coexist. A forbidden coexistence either makes a root good or creates
a circuit of rank four or five. These are the invariant exclusion tests
behind the algebra in Appendix~\ref{app:five-shared}.

The remaining configurations give a cover by two rank-three spans, a
three-span cover, a four-span configuration, or two quotient lines.
For the middle two, use the symmetric polar
$Q=\{z:|e^\top z|\leq1\ (e\in E)\}$.
At a vertex, color active roots by their signs. If neither sign supplies
two independent interpolating roots modulo an available rank-three
circuit, the active roots outside the shared two-space must all have
one unoriented direction (Lemma~\ref{lem:five-active}). On the relevant
polar slice, switching that direction would activate two nonparallel
roots and violate the restriction. This forces a realizable face or
vertex. Its squared norm is at most one by the original normalization.

This reasoning explains why the branch tolerates arbitrary extra rows:
the active-space restriction is applied to all exceptional roots, and
their signed hull contains every original root. Positive rescaling is
used only for dependence signs, ranks, and spans; it is never used to
claim that the original quadratic risk is unchanged.
Table~\ref{tab:invariants} indexes the structural proof obligations.

\begin{table}[t]
\centering\small
\caption{Invariants and where their full proofs are used.}
\label{tab:invariants}
\begin{tabularx}{\linewidth}{@{}lXX@{}}
\toprule
Stage & Invariant maintained & Full argument \\
\midrule
Realization & Positive weights; six rows; full rank & Lemma~\ref{lem:five-realization} \\
Exceptional reduction & Signed hull; $C\succeq I$, without centering & Lemma~\ref{lem:five-exceptional} \\
Shared circuits & No good crossing root or higher-rank circuit & Appendix~\ref{app:five-shared} \\
Polar faces & Failed signs force one transverse direction & Lemmas~\ref{lem:five-active}--\ref{lem:five-multiple} \\
Circuit planes & Minimal positive relations constrain intersections & Lemmas~\ref{lem:five-planes}--\ref{lem:five-shortcycles} \\
Graph to risk & Mean-zero vertex mixture of actual fits & \eqref{eq:five-realrisk}--\eqref{eq:five-lowrisk} \\
\bottomrule
\end{tabularx}
\end{table}

\section{Transfer, attainment, and the matching example}
\label{sec:five-six}
The root bound covers all normalized systems by
Proposition~\ref{prop:exhaustive}. To recover every regression instance,
one must also handle rows for which division by $\eta_i$ is invalid.
If $\eta_0=0$ and $\xi_0\ne0$, constrain $z$ to $\xi_0^\perp$ and
apply the four-dimensional five-row theorem there. Lift those rows and
add row zero with weight $M>0$. The resulting full-rank objective
converges, as $M\to\infty$, to the constrained fit and gives ratio
infimum at most two. Every finite $M$ is an admissible six-row design.
Pure-residual rows ($\xi_i=0$, $\eta_i\ne0$) only reduce the nonzero-root
mass to $a\leq1$; renormalizing and scaling gives excess at most
$6a/5\leq6/5$. Deficient observed ranks use the four-dimensional input,
the dimension-three boundary law, or exact certificates in ranks at most
two. Appendix~\ref{app:five-transfer} proves these transfers, including
the extension of the uniform four-dimensional input to arbitrary positive
evaluation weights. Thus the root construction proves the full upper
bound in Theorem~\ref{thm:five-six}.

For the lower bound take the seven feature rows
\[
 e_1,e_2,e_3,-e_1-e_2-e_3,\qquad e_4,e_5,-e_4-e_5,
\]
with evaluation weights $1/10$ on each of the first four and $1/5$ on
each of the last three. Put $h=(10,20,40,80,160)^\top$ and
$y_i=x_i^\top h+1$. Then $\wstar=h$, $\LD=1$, and the blocks have
dimensions three and two and total masses $2/5$ and $3/5$.
Any full-rank support of size at most six leaves one block with exactly
its dimension $r$ rows. Those rows interpolate at displacement one,
forcing the omitted simplex row to have displacement $-r$. Its block
excess is $a_r(r+r^2)/(r+1)=a_rr=6/5$. Retaining the full tetrahedron
with equal training weights and two triangle rows attains this value.
The translation by $h$ makes every deficient selected span incur excess
at least five, so minimum-norm rank loss cannot improve it. The exact
distance argument is in Appendix~\ref{app:five-lower}.
This proves both attainment for the example and $\Gamma_5(6)\geq11/5$.
It is the $3+2$ optimizer in \eqref{eq:landscape-five}, with the generic
translation replaced by an explicit vector.

\section{Certificate rigidity and the penultimate boundary}
\label{sec:certificate-summary}
The boundary theory gives a complementary, dimension-uniform account
of realizability. Define the residual-scaled features by
\begin{equation}
 r_i=x_i^\top\wstar-y_i,\qquad g_i=r_ix_i,
 \label{eq:gradient-atoms}
\end{equation}
so the normal equations give
\begin{equation}
 \sum_i p_ig_i=0.
 \label{eq:full-balance}
\end{equation}
For $U_q=\spann\{x_i:q_i>0\}$, exact minimum-norm recovery is equivalent to
\[
 \widehat w(q)=\wstar
 \quad\Longleftrightarrow\quad
 \sum_iq_ig_i=0\ \text{and}\ \wstar\in U_q.
\]
Stationarity is the first condition; orthogonality to the selected
Hessian's kernel is the second. A full-span certificate enforces the
stronger condition $U_q=U$. Let $\kappa(D)$ be its minimum support size.
Theorems~\ref{thm:exact-certificate}--\ref{thm:maximal-rigidity} give
\[
 m\leq\kappa(D)\leq m+r\leq2m\leq2d
 \quad(m=\dim U\geq1,\ r=\dim\spann\{g_i:g_i\ne0\}),
\]
with $\kappa(D)=1$ when $m=0$.
Steinitz compression \citep{steinitz1916,ivanov2024quantitative} and
maximal positive-basis rigidity imply that $\kappa(D)=2m$ holds exactly
when the nonzero features lie on $m$ independent lines and all have
nonzero gradients. Both directions, including extra mixed features and
zero-gradient rows, are proved in Appendix~\ref{app:proofs}.

On this rigid branch the risk is $C_0+\sum_{j=1}^m L_j(t_j)$, with
$L_j(t)=\sum_i a_i(t-z_i)^2$, $a_i>0$. If
$A_j=\sum_i a_i$ and $\mu_j=\sum_i a_i z_i/A_j$, the exact one-point
penalty is
\[
 \delta_j=A_j\min_i(z_i-\mu_j)^2\leq L_j^\star.
\]
Two points recover $\mu_j$ exactly. Equality in the penalty bound, for
positive block loss, occurs precisely at two equidistant target levels
of equal effective mass (Theorem~\ref{thm:one-dimensional-profile}).
If $\kappa(D)\leq2d-1$, selection is exact; otherwise, compress the least
costly of the $d$ lines to one point and retain two points on each other
line. The resulting positive-weight full-span fit satisfies
\[
 \LDS{2d-1}\leq\LD+\min_j\delta_j
 \leq\LD+\frac1d\sum_jL_j^\star\leq(1+1/d)\LD.
\]
The uniform coordinate pairs $(e_j,0),(e_j,2)$ attain equality, giving
$\Gamma_d(2d-1)=1+1/d$ for every $d\geq1$.
All upper bounds here are attained by explicit selections.

The equality analysis is stronger than the value alone: every extremizer
has maximal certificate complexity, zero zero-feature loss, equal positive
block losses, factor-two extremal blocks, and lower bounds on all canonical
rank-deficient projection selections. The complete necessary conditions,
one-dimensional profile, and instance-dependent bound remain in
Theorems~\ref{thm:one-dimensional-profile}--\ref{thm:sharp-boundary}.
These are necessary extremal conditions, not an asserted sufficient
classification of every boundary extremizer.

\section{Conclusion}
The exact law $\Gamma_d(n)=3-n/d$ holds throughout the upper-half interval,
with a complete Lean~4 verification of its dataset-level statement.
Balanced anchors, positive-weight lifting, and independent-line compression
close the dimension induction. At $(5,6)$, the separate circuit analysis
proves the block prediction $11/5$ for arbitrary configurations.
These results connect dimension-uniform budget laws with the finer geometry
needed below the upper-half interval.

%% file: appendix/geometry.tex
For a selection \(q\in\Delta_N\), let
\(L_q(w)=\sum_{i=1}^Nq_i(\langle w,x_i\rangle-y_i)^2\) denote the selected
objective, and write
\[
  U_q=\spann\{x_i:q_i>0\},
  \qquad
  h(q)=\sum_{i=1}^Nq_ig_i.
\]
The quantities have distinct roles: \(h(q)\) is the selected
half-gradient at the full-data solution, so
\(\nabla L_q(\wstar)=2h(q)\), whereas \(U_q\) determines the minimum-norm
tie-breaking rule.

\begin{theorem}[Exact-Selection Certificate Theorem]
\label{thm:exact-certificate}
For every \(q\in\Delta_N\),
\begin{equation}
  \widehat w(q)=\wstar
  \quad\Longleftrightarrow\quad
  h(q)=0
  \ \text{and}\ 
  \wstar\in U_q.
  \label{eq:exact-characterization}
\end{equation}
In particular, the full-observed-span conditions
\begin{equation}
  h(q)=0,
  \qquad
  U_q=U
  \label{eq:full-span-certificate}
\end{equation}
imply \(\widehat w(q)=\wstar\) and
\(L_D(\widehat w(q))=\LD\).
\end{theorem}

\begin{proof}
The selected objective depends only on the projection of \(w\) onto \(U_q\),
so its minimum-norm minimizer belongs to \(U_q\).  Consequently,
\(\widehat w(q)=\wstar\) implies \(\wstar\in U_q\).  Stationarity at
\(\wstar\) also gives
\(
  0=\nabla L_q(\wstar)=2h(q)
\), proving necessity.

Conversely, if \(h(q)=0\), then \(\wstar\) minimizes the selected convex
quadratic.  Because all weights on \(\supp(q)\) are positive, its Hessian has
kernel \(U_q^\perp\), and the complete minimizer set is
\(\wstar+U_q^\perp\).  If \(\wstar\in U_q\), it is orthogonal to this kernel
and is therefore the unique minimum-norm point of the affine minimizer set.
Finally, \(\wstar\in U\) by the full-data minimum-norm convention, so
\cref{eq:full-span-certificate} is sufficient.
\end{proof}

The full-span condition yields a geometric certificate complexity controlled
by intrinsic rank.

\begin{definition}[Full-span certificate complexity]
\label{def:certificate-complexity}
The full-span certificate complexity of \(D\) is
\begin{equation}
  \kappa(D)
  =
  \min\left\{
    |\supp(q)|:
    q\in\Delta_N,\ h(q)=0,\ U_q=U
  \right\}.
  \label{eq:kappa}
\end{equation}
The feasible set is nonempty because the original weights \(p\) satisfy
\cref{eq:full-balance} and have feature span \(U\).  Since support sizes take
values in the finite set \(\{1,\ldots,N\}\), the displayed minimum is attained.
\end{definition}

\ifdefined\NatureAfterCertificateDefinition
  \NatureAfterCertificateDefinition
\fi

The first complexity bound depends on the observed rank and the intrinsic
dimension of the nonzero gradient atoms rather than on the ambient dimension
alone.

\begin{theorem}[Intrinsic-Dimension Certificate Bound]
\label{thm:intrinsic-bound}
Let
\[
  m=\dim U,
  \qquad
  V=\spann\{g_i:g_i\neq0\},
  \qquad
  r=\dim V.
\]
If \(m=0\), then \(\kappa(D)=1\).  If \(m\geq1\), then
\begin{equation}
  m
  \leq
  \kappa(D)
  \leq
  m+r
  \leq
  2m
  \leq
  2d.
  \label{eq:intrinsic-bound}
\end{equation}
In particular, \(r<m\) implies \(\kappa(D)\leq2m-1\), and
\(1\leq m<d\) implies \(\kappa(D)\leq2d-2\).
\end{theorem}

The proof first compresses the active gradient balance to at most \(2r\)
atoms by Steinitz's theorem and then adds at most \(m-r\) zero-gradient
features to complete the observed span.  Strictly positive weights are kept
on the common support.  The cases \(r=0\) and \(m=0\) require separate, short
arguments.  A complete proof is given in
\ref{app:intrinsic-bound-proof}.

\ifdefined\NatureAfterIntrinsicBound
  \NatureAfterIntrinsicBound
\fi

Equality in the largest intrinsic bound has a rigid converse when every
nonzero feature carries a nonzero gradient.

\begin{theorem}[Maximal Certificate-Complexity Rigidity Theorem]
\label{thm:maximal-rigidity}
Assume \(m=\dim U\geq1\).  Then the following statements are equivalent.

\begin{enumerate}[leftmargin=*,label=(\roman*)]
  \item \(\kappa(D)=2m\).
  \item There is a basis \(v_1,\ldots,v_m\) of \(U\) such that every
  nonzero feature satisfies
  \begin{equation}
    x_i\in\bigcup_{j=1}^m\R v_j
    \qquad\text{and}\qquad
    g_i\neq0.
    \label{eq:rigid-lines}
  \end{equation}
\end{enumerate}

Under these conditions, each line contains gradients of both signs and an
inclusion-minimal positive basis of the active gradients has the form
\begin{equation}
  \{a_jv_j,-b_jv_j:j\in[m]\},
  \qquad a_j,b_j>0.
  \label{eq:paired-positive-basis}
\end{equation}
Moreover, writing each nonzero feature as
\(x_i=\xi_iv_{j(i)}\), with \(\xi_i\neq0\), yields
\begin{equation}
  L_D(w)
  =
  C_0+\sum_{j=1}^mL_j(\langle w,v_j\rangle),
  \qquad
  C_0=\sum_{i:x_i=0}p_iy_i^2,
  \label{eq:line-decomposition}
\end{equation}
where
\begin{equation}
  L_j(t)
  =
  \sum_{i:j(i)=j}p_i(\xi_it-y_i)^2.
  \label{eq:line-block}
\end{equation}
Consequently, \(\kappa(D)=2d\) holds if and only if \(U=\R^d\) and
\cref{eq:rigid-lines} holds for a basis of \(\R^d\).
\end{theorem}

For the forward direction, \cref{thm:intrinsic-bound} first forces the
active gradients to span \(U\).  For a full-dimensional interior subset
\(B\), relative interior gives \(\operatorname{pos}(B)=U\).  If a proper
subset positively spanned \(U\), a strictly positive zero representation,
after normalization, would give a smaller full-span certificate.  Thus \(B\)
is inclusion-minimal, and \(\kappa(D)=2m\) forces the maximal positive-basis
size \(2m\), so \(B\) consists of opposite pairs.  Any additional mixed
gradient or any nonzero zero-gradient feature
would explicitly produce a certificate with at most \(2m-1\) points.  For
the converse, a spanning certificate must represent every independent line,
and balance on a line cannot use only one nonzero gradient; hence it needs at
least two points per line.  The full proof, including the positive
coefficients, zero features, and \(m=1\), is in
\ref{app:rigidity-proof}.

%% file: appendix/one_dimensional.tex
The rigid geometry in \cref{thm:maximal-rigidity} reduces the selection problem
to a precise one-dimensional calculation.  The following result gives the
complete profile and characterizes equality in the factor-two bound.

\begin{theorem}[Exact One-Dimensional Compression Profile]
\label{thm:one-dimensional-profile}
Let
\[
  F(t)=\sum_{i=1}^Ma_i(t-z_i)^2,
  \qquad
  M\geq1,
  \qquad
  a_i>0,
\]
and define
\[
  A=\sum_{i=1}^Ma_i,
  \qquad
  \mu=\frac{1}{A}\sum_{i=1}^Ma_iz_i,
  \qquad
  F^\star=F(\mu).
\]
For \(\alpha\in\Delta_M\), let
\(\widehat t(\alpha)=\argmin_t\sum_i\alpha_i(t-z_i)^2
=\sum_i\alpha_i z_i\), and set
\[
  F_s^\star
  =
  \min_{\substack{\alpha\in\Delta_M\\|\supp(\alpha)|\leq s}}
  F(\widehat t(\alpha)).
\]
Then
\begin{equation}
  F_2^\star=F^\star,
  \qquad
  F_1^\star=F^\star+\delta(F),
  \qquad
  \delta(F)=A\min_{i\in[M]}(z_i-\mu)^2,
  \label{eq:one-dimensional-profile}
\end{equation}
and
\begin{equation}
  0\leq\delta(F)\leq F^\star.
  \label{eq:delta-bound}
\end{equation}
Consequently, \(F_1^\star\leq2F^\star\).

If \(F^\star>0\), equality \(F_1^\star=2F^\star\) holds if and only if
there is a \(c>0\) such that every \(z_i\) belongs to
\(\{\mu-c,\mu+c\}\) and
\begin{equation}
  \sum_{i:z_i=\mu-c}a_i
  =
  \sum_{i:z_i=\mu+c}a_i
  =
  \frac{A}{2}.
  \label{eq:factor-two-equality}
\end{equation}
If \(F^\star=0\), all \(z_i=\mu\), and the one-point selection is exact.
\end{theorem}

\begin{proof}
Completing the square gives
\begin{equation}
  F(t)=F^\star+A(t-\mu)^2.
  \label{eq:variance-decomposition}
\end{equation}
The mean \(\mu\) either equals a support value or lies between two support
values.  In the latter case, a strictly positive convex combination of the
two bracketing values equals \(\mu\).  Hence a selection of at most two
support values attains \(F^\star\).

A one-point selection outputs one of the values \(z_i\), so
\cref{eq:variance-decomposition} gives the exact formula for \(F_1^\star\).
Moreover,
\[
  \sum_{i=1}^M\frac{a_i}{A}\,
  A(z_i-\mu)^2
  =
  F^\star.
\]
The minimum of the displayed terms is at most their weighted average, which
proves \cref{eq:delta-bound}.  Equality holds precisely when every term with
positive weight equals \(F^\star\).  If \(F^\star>0\), this places every
support value at the same positive distance \(c=\sqrt{F^\star/A}\) from
\(\mu\).  Both signs must occur, and the defining mean identity then forces
equal total \(a\)-mass on the two levels.  The converse is immediate.  If
\(F^\star=0\), positivity of the \(a_i\)'s forces every \(z_i=\mu\).
\end{proof}

We now translate the profile back to a line block.  Under
\cref{eq:line-block}, for \(i\) on line \(j\) define
\begin{equation}
  a_i=p_i\xi_i^2,
  \qquad
  z_i=\frac{y_i}{\xi_i},
  \qquad
  A_j=\sum_{i:j(i)=j}a_i,
  \qquad
  \mu_j=\frac{1}{A_j}\sum_{i:j(i)=j}a_iz_i.
  \label{eq:effective-targets}
\end{equation}
Then
\[
  L_j(t)=\sum_{i:j(i)=j}a_i(t-z_i)^2,
  \qquad
  L_j^\star=L_j(\mu_j),
\]
and the nonsingular Gram matrix of \(v_1,\ldots,v_m\) makes the coordinate
map
\[
  w\longmapsto
  \bigl(\langle w,v_1\rangle,\ldots,\langle w,v_m\rangle\bigr)
\]
invertible on \(U\).  The blocks can therefore be minimized independently.
Since \(\wstar\in U\) minimizes the full objective,
\begin{equation}
  \mu_j=\langle\wstar,v_j\rangle\quad(j\in[m]),
  \qquad
  \LD=C_0+\sum_{j=1}^mL_j^\star.
  \label{eq:block-optimum-decomposition}
\end{equation}
and we write
\begin{equation}
  \delta_j
  =
  A_j\min_{i:j(i)=j}(z_i-\mu_j)^2.
  \label{eq:block-delta}
\end{equation}
Zero features do not enter \cref{eq:effective-targets}; their contribution is
the constant \(C_0\) in \cref{eq:line-decomposition}.

The abstract weights in \cref{thm:one-dimensional-profile} correspond to
actual regression selection weights.  If \(\alpha_i\) are desired effective
weights within a line, take the regression weights proportional to
\(\alpha_i/\xi_i^2\).  Local selections on different lines can be multiplied
by arbitrary positive line masses and then normalized globally.  When at
least one nonzero feature is retained on each of the \(m\) independent lines,
the selected design spans \(U\).  The chosen blockwise minimizers then specify
a unique vector \(w_U\in U\); the complete set of selected ERM minimizers is
\(w_U+U^\perp\), and minimum-norm tie-breaking returns \(w_U\).  In the
full-rank applications below, \(U=\R^d\), so the selected ERM minimizer is unique.

%% file: appendix/sharp_boundary.tex
The certificate complexity separates instances admitting a full-span exact
certificate at budget \(2d-1\) from the maximal-complexity branch that requires
additional control.  On the latter branch,
\cref{thm:one-dimensional-profile} supplies an explicit one-point compression bound
with a data-dependent penalty before any worst-case averaging is used.

\begin{theorem}[Instance-Dependent Penultimate-Budget Bound]
\label{thm:instance-bound}
For every finite weighted regression instance in \(\R^d\), exactly one of
the following two alternatives holds.
\begin{enumerate}[leftmargin=*,label=(\alph*)]
  \item If \(\kappa(D)\leq2d-1\), then
  \begin{equation}
    \LDS{2d-1}=\LD.
    \label{eq:exact-small-kappa}
  \end{equation}
  \item If \(\kappa(D)=2d\), use the \(d\)-line decomposition in
  \cref{thm:maximal-rigidity} and the penalties \(\delta_j\) from
  \cref{eq:block-delta}.  Then
  \begin{equation}
    \LDS{2d-1}
    \leq
    \LD+\min_{j\in[d]}\delta_j.
    \label{eq:instance-dependent-bound}
  \end{equation}
\end{enumerate}
In the second case,
\begin{equation}
  \LDS{2d-1}
  \leq
  \LD+\min_j\delta_j
  \leq
  \LD+\min_jL_j^\star
  \leq
  \LD+\frac{1}{d}\sum_{j=1}^dL_j^\star
  \leq
  \left(1+\frac{1}{d}\right)\LD.
  \label{eq:boundary-chain}
\end{equation}
\end{theorem}

\begin{proof}
Theorem~\ref{thm:intrinsic-bound} gives \(\kappa(D)\leq2d\), so the two
integer-valued alternatives are exhaustive and disjoint.  In case (a), a
minimizing certificate in \cref{eq:kappa} is a feasible
selection with at most \(2d-1\) support points.  Theorem
\ref{thm:exact-certificate} recovers \(\wstar\).  Since \(\LD\) is the
minimum of the original risk over all predictors, the selected infimum equals
the full optimum.

In case (b), \cref{thm:maximal-rigidity} gives \(U=\R^d\) and the
decomposition
\[
  \LD=C_0+\sum_{j=1}^dL_j^\star.
\]
Choose \(j_0\) minimizing \(\delta_j\).  Use the two-point exact compression
from \cref{thm:one-dimensional-profile} on every line \(j\neq j_0\), and use
the best one-point compression on line \(j_0\).  The implementation following
\cref{eq:block-delta} gives a genuine regression selection with positive
weights and at most \(2(d-1)+1=2d-1\) support points.  Its original-data risk
is exactly
\[
  C_0+\sum_{j=1}^dL_j^\star+\delta_{j_0}
  =
  \LD+\min_j\delta_j,
\]
which proves \cref{eq:instance-dependent-bound}.  The remaining inequalities
in \cref{eq:boundary-chain} follow from
\(\delta_j\leq L_j^\star\), the minimum--average inequality, and
\(\sum_jL_j^\star\leq\LD\).
\end{proof}

The universal factor is sharp.  Equality in the universal bound forces
equality at every stage of \cref{eq:boundary-chain} together with lower bounds
for every canonical rank-deficient projection selection.

\begin{theorem}[Sharp Boundary Law and Extremal Rigidity]
\label{thm:sharp-boundary}
For every integer \(d\geq1\),
\begin{equation}
  \boxed{\GammaDN{2d-1}=1+\frac{1}{d}.}
  \label{eq:sharp-boundary}
\end{equation}

Moreover, suppose an instance \(D\) with \(\LD>0\) attains the universal
ratio:
\begin{equation}
  \frac{\LDS{2d-1}}{\LD}=1+\frac{1}{d}.
  \label{eq:extremal-instance}
\end{equation}
Then all of the following conditions are necessary.
\begin{enumerate}[leftmargin=*,label=(\roman*)]
  \item \(\kappa(D)=2d\).  Hence \(U=\R^d\), every nonzero feature has a
  nonzero gradient, and all nonzero features lie on \(d\) independent lines
  as in \cref{thm:maximal-rigidity}.
  \item The zero-feature loss vanishes and all line blocks have equal
  positive optimal loss:
  \begin{equation}
    C_0=0,
    \qquad
    L_1^\star=\cdots=L_d^\star=\frac{\LD}{d}>0.
    \label{eq:equal-block-loss}
  \end{equation}
  Equivalently, every support point with \(x_i=0\) has \(y_i=0\).
  \item Every block is a nondegenerate factor-two one-point extremizer:
  \begin{equation}
    \delta_j=L_j^\star
    \qquad (j\in[d]).
    \label{eq:all-block-extremal}
  \end{equation}
  Equivalently, its effective targets take exactly two values
  \(\mu_j-c_j\) and \(\mu_j+c_j\), with \(c_j>0\), and the two levels have
  equal total effective mass \(A_j/2\).
  \item For every nonempty proper \(J\subset[d]\), let
  \(W_J=\spann\{v_j:j\in J\}\) and let \(P_J\) be the Euclidean orthogonal
  projection onto \(W_J\).  Then
  \begin{equation}
    \sum_{k\notin J}
    A_k\bigl(\langle P_J\wstar,v_k\rangle-\mu_k\bigr)^2
    \geq
    \frac{\LD}{d}.
    \label{eq:canonical-projection-bound}
  \end{equation}
\end{enumerate}
\end{theorem}

\begin{proof}
The upper bound in \cref{eq:sharp-boundary} is
\cref{thm:instance-bound}.  For the lower bound, consider the uniform
\(2d\)-point instance
\begin{equation}
  D_d^{\mathrm{pair}}
  =
  \operatorname{Unif}
  \{(e_j,0),(e_j,2):j\in[d]\}.
  \label{eq:pair-instance}
\end{equation}
Its full-data solution is \(\wstar=(1,\ldots,1)\), and its optimal risk is
one.  For a selected minimum-norm ERM, write \(t_j\) for coordinate \(j\).
The original-risk contribution of that coordinate is
\begin{equation}
  \frac{1}{2d}\bigl(t_j^2+(t_j-2)^2\bigr)
  =
  \frac{1}{d}+\frac{(t_j-1)^2}{d}.
  \label{eq:coordinate-risk}
\end{equation}
With at most \(2d-1\) positive-weight points, some coordinate contains at
most one selected endpoint.  If one endpoint is selected, \(t_j\) equals
zero or two; if neither is selected, minimum-norm tie-breaking gives
\(t_j=0\).  That coordinate contributes \(2/d\), while every other
coordinate contributes at least \(1/d\).  Thus every feasible selection has
risk at least \(1+1/d\).  Selecting equal-weight pairs on \(d-1\)
coordinates and one endpoint on the last coordinate attains this value.
Therefore both the instance infimum and the distributional supremum in
\cref{eq:sharp-boundary} are attained.

Now assume \cref{eq:extremal-instance}.  If
\(\kappa(D)\leq2d-1\), \cref{eq:exact-small-kappa} would give ratio one.
Hence \(\kappa(D)=2d\), proving (i).  Combining the assumed equality with
\cref{eq:boundary-chain} gives
\[
  \left(1+\frac1d\right)\LD
  \leq
  \LD+\min_j\delta_j
  \leq
  \LD+\min_jL_j^\star
  \leq
  \LD+\frac1d\sum_jL_j^\star
  \leq
  \left(1+\frac1d\right)\LD.
\]
Every inequality is therefore an equality.  The last equality gives
\(C_0=0\); equality in the minimum--average step makes all
\(L_j^\star=\LD/d\); and
\(\delta_j\leq L_j^\star\), together with equality of their minima, gives
\(\delta_j=L_j^\star\) for every \(j\).  This proves (ii) and (iii), including
the two-level characterization from
\cref{thm:one-dimensional-profile}.

Finally, fix a nonempty proper \(J\subset[d]\) and apply exact two-point
compression only on the lines in \(J\).  The selected objective is minimized
by \(\wstar+W_J^\perp\), so its minimum-norm solution is \(P_J\wstar\).
This uses at most \(2|J|\leq2d-2\) support points and is feasible at budget
\(2d-1\).  Variance decomposition on each line gives
\[
  L_D(P_J\wstar)
  =
  \LD+
  \sum_{k\notin J}
  A_k\bigl(\langle P_J\wstar,v_k\rangle-\mu_k\bigr)^2.
\]
No feasible selection can have risk below the infimum in
\cref{eq:extremal-instance}, which proves
\cref{eq:canonical-projection-bound}.

Under the extremal hypothesis, every budget-\((2d-1)\) selection therefore
has risk at least \(\LD+\LD/d\).  Thus no rank-changing sequence of weights
can approach the full optimum at that budget.  In contrast,
\(\kappa(D)=2d\) supplies an actual \(2d\)-point exact certificate, so exact
recovery is attained at the universal recovery threshold.
\end{proof}

%% file: appendix/proofs.tex
\subsection{Convex-geometric preliminaries}
\label{app:geometry-preliminaries}

We record the two external convex-geometric inputs and the positivity facts
used below.  If a finite set \(G\) spans a linear space \(V\) and zero is a
strictly positive convex combination of all points in \(G\), then
\[
  0\in\relint_V\conv(G).
\]
Indeed, a supporting functional at zero would be nonnegative on \(G\).
Applying it to the strictly positive combination forces it to vanish on every
point of \(G\), and hence on \(V\), a contradiction.  Conversely, if
\(0\in\relint\conv(G)\), then zero has a convex representation with a
strictly positive coefficient on every point of \(G\).  This is the standard
identity that a linear map sends the relative interior of a simplex onto the
relative interior of its image.

Carath\'eodory's theorem gives a small convex representation, but such a
representation need not retain the full-dimensional relative-interior
property required here \cite{scherk1966caratheodory}.

We use Steinitz's theorem in the following relative-interior form: if zero is
in the relative interior of the convex hull of a finite subset of an
\(r\)-dimensional linear space, then a subset of at most \(2r\) points still
has an \(r\)-dimensional convex hull containing zero in its relative interior
\cite{steinitz1916,ivanov2024quantitative}.

We also use the extremal structure of a positive basis.  An
inclusion-minimal positive spanning set in an \(m\)-dimensional space has at
most \(2m\) vectors; if it has exactly \(2m\), then, after rescaling and
reindexing, it is
\[
  \{a_jv_j,-b_jv_j:j\in[m]\},
  \qquad a_j,b_j>0,
\]
for a basis \(v_1,\ldots,v_m\).  This is the maximal-cardinality case of the
positive-basis structure theorem; see
\cite{davis1954positive,regis2016positive} and the discussion preceding
Theorem~2 of \cite{naevdal2019positive}.

\subsection{Proof of the intrinsic-dimension certificate bound}
\label{app:intrinsic-bound-proof}

\begin{proof}[Proof of \cref{thm:intrinsic-bound}]
Because \(g_i\) is a scalar multiple of \(x_i\), \(V\subseteq U\), and hence
\(r\leq m\).  If \(m=0\), all features and gradients are zero.  Any one-point
selection is a full-span certificate, while every point of \(\Delta_N\) has
nonempty support.  Thus \(\kappa(D)=1\).

Assume \(m\geq1\).  Every collection of features spanning \(U\) contains at
least \(m\) points, so \(\kappa(D)\geq m\).  If \(r=0\), all gradients are
zero.  Choose \(m\) features forming a basis of \(U\) and give them arbitrary
positive normalized weights.  This is a certificate with \(m\) points, so the
claimed upper bound holds.

It remains to consider \(r\geq1\).  Let
\[
  A=\{i:g_i\neq0\}.
\]
The active gradients span \(V\).  By \cref{eq:full-balance}, positivity of the
original masses, and \(g_i=0\) outside \(A\),
\[
  \sum_{i\in A}p_ig_i=0.
\]
Since \(A\neq\varnothing\), define the normalized active masses
\[
  \widetilde p_i
  =
  \frac{p_i}{\sum_{k\in A}p_k}
  \quad(i\in A).
\]
They are strictly positive, sum to one, and satisfy
\(\sum_{i\in A}\widetilde p_i g_i=0\).  The first preliminary fact therefore
gives
\[
  0\in\relint_V\conv\{g_i:i\in A\}.
\]
Steinitz's theorem gives \(T\subseteq A\), with \(|T|\leq2r\), such that the
gradients in \(T\) span \(V\) and zero lies in the relative interior of their
convex hull.  Thus there are coefficients \(\beta_i>0\), \(i\in T\), summing
to one and satisfying \(\sum_{i\in T}\beta_ig_i=0\).  Since \(g_i\) is a
nonzero scalar multiple of \(x_i\) on \(A\), the features indexed by \(T\)
also span \(V\).

If \(r=m\), set \(q_i=\beta_i\) on \(T\) and zero elsewhere.  This is a
full-span certificate with at most \(2m=m+r\) support points.

If \(r<m\), the active features span exactly \(V\).  Select a set
\(Z\subseteq[N]\setminus A\) of \(m-r\) indices whose feature cosets form a
basis of \(U/V\).  Every \(z\in Z\) has \(g_z=0\).  Choose positive
coefficients \(\eta_z\), \(z\in Z\), summing to one, and fix
\(\varepsilon\in(0,1)\).  Define
\[
  q_i=(1-\varepsilon)\beta_i\quad(i\in T),
  \qquad
  q_z=\varepsilon\eta_z\quad(z\in Z),
\]
with all other weights zero.  Its gradient sum is zero, its features span
\(U\), and its support size is at most
\[
  2r+(m-r)=m+r.
\]
This proves \(\kappa(D)\leq m+r\leq2m\leq2d\).  If \(r<m\), then
\(m+r\leq2m-1\); if \(1\leq m<d\), then \(2m\leq2d-2\), proving the stated
consequences.
\end{proof}

\subsection{Proof of maximal certificate-complexity rigidity}
\label{app:rigidity-proof}

\begin{proof}[Proof of \cref{thm:maximal-rigidity}]
Assume first that \(\kappa(D)=2m\).  Theorem~\ref{thm:intrinsic-bound} gives
\(2m\leq m+r\leq2m\), so \(r=m\) and \(V=U\).  As in the preceding proof,
the positive full-data balance puts zero in the relative interior of the
convex hull of the active gradients.  Steinitz's theorem supplies a
full-dimensional interior subset \(B\) of at most \(2m\) active gradients.
Its strictly positive balancing coefficients and its feature span make it a
certificate.  Moreover,
\[
  0\in\relint_U\conv(B)
  \quad\Longrightarrow\quad
  \operatorname{pos}(B)=U,
\]
because a relative neighborhood of zero in \(U\) lies in \(\conv(B)\).
If a proper subset \(C\subset B\) positively spanned \(U\), then
\(-c\in\operatorname{pos}(C)\) for every \(c\in C\).  Summing these
representations would give a zero combination with a strictly positive
coefficient on every point of \(C\).  After normalization, these coefficients
would balance the corresponding active gradients.  Because every active
gradient is a nonzero scalar multiple of its feature, the corresponding
features also span \(U\).  Thus \(C\) would yield a smaller full-span
certificate.  Maximality of \(\kappa(D)\) rules this out and also forces
\(|B|=2m\).  Hence \(B\) is an inclusion-minimal positive spanning set of
maximal cardinality.  By the structure theorem in
\ref{app:geometry-preliminaries}, after reindexing,
\begin{equation}
  B=\{a_jv_j,-b_jv_j:j\in[m]\},
  \qquad a_j,b_j>0,
  \label{eq:appendix-maximal-positive-basis}
\end{equation}
for a basis \(v_1,\ldots,v_m\) of \(U\).

We next rule out a mixed active gradient.  Suppose an active gradient outside
the union of the basis lines has the expansion
\[
  h=\sum_{j\in J}c_jv_j,
  \qquad |J|\geq2,
  \qquad c_j\neq0.
\]
Give \(h\) coefficient one.  For each \(j\in J\), choose from
\cref{eq:appendix-maximal-positive-basis} the vector with sign opposite to
\(c_j\), with the unique positive coefficient that cancels \(c_jv_j\).  For
each \(k\notin J\), choose both vectors of the \(k\)th pair with positive
coefficients that balance the pair.  The resulting positive combination sums
to zero, spans \(U\), and uses
\[
  1+|J|+2(m-|J|)=2m-|J|+1\leq2m-1
\]
active points.  After normalization it is a smaller certificate, a
contradiction.  Hence every active gradient, and therefore its feature, lies
on one of the lines \(\R v_j\).

Now suppose that an index \(\ell\) has \(x_\ell\neq0\) but \(g_\ell=0\).
Write \(x_\ell=\sum_jc_jv_j\) and choose \(k\) with \(c_k\neq0\).  Select
\(\ell\) and both members of every pair in
\cref{eq:appendix-maximal-positive-basis} except the \(k\)th pair.  Give the
pairs positive balancing coefficients and give \(\ell\) any positive
coefficient.  The gradients balance, while \(x_\ell\) together with
\(\{v_j:j\neq k\}\) spans \(U\).  This certificate uses
\(1+2(m-1)=2m-1\) points, again a contradiction.  Thus every nonzero feature
has a nonzero gradient, proving the geometric condition in (ii).

Conversely, assume (ii).  For each line write every gradient on that line as
\(g_i=\gamma_iv_j\), where \(\gamma_i\neq0\).  Each line is represented
because the features span \(U\).  Resolving the full-data balance
\cref{eq:full-balance} in the basis \(v_1,\ldots,v_m\) gives
\[
  \sum_{i:x_i\in\R v_j\setminus\{0\}}p_i\gamma_i=0
  \qquad (j\in[m]).
\]
Hence every line contains gradients of both signs.  Choose one positive and
one negative gradient from each line, balance each chosen pair with strictly
positive coefficients, combine the \(m\) pairwise balances, and normalize.
This gives a full-span certificate with \(2m\) points.

For the reverse inequality, let \(q\) be any full-span certificate.  Since all
nonzero selected features lie on the independent basis lines, its support must
represent every line.  Gradient balance also holds separately on every line.
With positive selection weights and nonzero \(\gamma_i\), a single selected
point cannot balance a line, so at least two points are required on each of
the \(m\) lines.  Therefore \(|\supp(q)|\geq2m\), proving
\(\kappa(D)=2m\).

It remains only to record the loss decomposition.  Under either equivalent
condition, write every nonzero feature as \(x_i=\xi_iv_{j(i)}\), with
\(\xi_i\neq0\).  Zero features contribute the constant
\(C_0=\sum_{i:x_i=0}p_iy_i^2\), while direct substitution gives
\[
  L_D(w)=C_0+
  \sum_{j=1}^m\sum_{i:j(i)=j}
  p_i\bigl(\xi_i\langle w,v_j\rangle-y_i\bigr)^2.
\]
The coordinate map from \(U\) to
\((\langle w,v_1\rangle,\ldots,\langle w,v_m\rangle)\) is invertible because
the Gram matrix of the basis is nonsingular; orthogonality is not required.
This proves \cref{eq:line-decomposition}.  Finally,
\(\kappa(D)=2d\) forces \(m=d\) by
\cref{eq:intrinsic-bound}, and the ambient equivalence follows from the result
just proved.
\end{proof}

%% file: appendix/five_dimensional_proof.tex
\subsection{Normalized roots and realizable training outputs}
\label{app:five-interfaces}
Throughout the geometric argument, let \(V=\{v_i\}_{i=1}^N\) consist of
nonzero vectors in \(\R^5\), with
\begin{equation}
 \lambda_i>0,\quad \sum_i\lambda_i=1,\quad
 \sum_i\lambda_i v_i=0,\quad
 \sum_i\lambda_i v_iv_i^\top=I_5.                 \label{eq:five-normal}
\end{equation}
Write \(\delta_6(V)\) for the infimum of \(\|z\|^2\) over the unique
minimizers of
\(\sum_{i\in S}w_i(v_i^\top z-1)^2\), where \(|S|\leq6\),
\(w_i>0\), and \(\spann\{v_i:i\in S\}=\R^5\).  By
\eqref{eq:five-normal}, their evaluation risk is \(1+\|z\|^2\).
Every root belongs to a positive circuit, and the roots positively generate
\(\R^5\).  A positive circuit is the support of an extreme ray of
\(\{c\geq0:\sum_i c_i v_i=0\}\); one of rank \(r\) has \(r+1\) roots.
Decomposing a positive zero relation into extreme rays proves these facts.

\begin{lemma}[Circuit--interpolation realization]
\label{lem:five-realization}
Let \(A\) be a positive circuit of rank \(r\), with span \(P\), and let
\(J\) consist of \(5-r\) roots independent modulo \(P\).  If
\[
 v_j^\top z=1\ (j\in J),\qquad v_i^\top z\leq1\ (i\in A),
\]
then \(z\) is a limit of admissible six-row training outputs.
\end{lemma}
\begin{proof}
Choose \(c_i>0\) with \(\sum_{i\in A}c_iv_i=0\).  There is a unique
\(z_0\in P^\perp\) interpolating \(J\).  For
\(z_\epsilon=(1-\epsilon)z+\epsilon z_0\), the evaluations on \(A\) are
strictly less than one.  Set
\(w_i=c_i/(1-v_i^\top z_\epsilon)\) on \(A\) and give \(J\) arbitrary
positive weights.  The normal equations hold because the interpolated rows
contribute zero and the remaining gradient is \(-\sum c_iv_i=0\).
The selected span is full and \(|A|+|J|=6\).  Let \(\epsilon\downarrow0\).
\end{proof}

In particular, \(|v_i^\top z|\leq1\) for every root implies
\(\|z\|^2\leq1\).  For a rank-three circuit, two roots evaluating to one
and independent in its quotient therefore suffice for \(\delta_6\leq1\).

Call a root \emph{good} if it is a strictly positive combination of five
other independent roots.  Let \(E\) be the set of remaining, or
\emph{exceptional}, roots and set \(K=\conv\{\pm v_i\}\).
\begin{lemma}[Boundary and comparison second moment]
\label{lem:five-exceptional}
If a good root lies on \(\partial K\), then \(\delta_6\leq1\).
Otherwise
\begin{equation}
 K=\conv\{\pm e:e\in E\},\qquad
 C=\sum_{e\in E}\mu_e ee^\top\succeq I_5,
 \quad \mu_e\geq0,\quad\sum_e\mu_e=1             \label{eq:five-comparison}
\end{equation}
for some \(\mu\), and \(E\) spans \(\R^5\).
\end{lemma}
\begin{proof}
Write the good boundary root as \(g=\sum_{j=1}^5a_jv_j\), \(a_j>0\),
and choose a supporting functional \(z\) with \(g^\top z=1\) and
\(|v_i^\top z|\leq1\).  Put \(a=\sum a_j\geq1\).  If \(a=1\), all five
independent generators evaluate to one.  If \(a>1\), take \(c\in(1/a,1)\)
and the vector evaluating to \(c\) on the five generators.  A short segment
from \(z\) toward this vector gives \(z_\epsilon\) with generator
evaluations below one and \(g^\top z_\epsilon>1\).  The weights
\[
 w_g=(g^\top z_\epsilon-1)^{-1},\qquad
 w_j=a_j/(1-v_j^\top z_\epsilon)
\]
give a full-rank six-row fit.  Its limit has squared norm at most one.
In the other case every vertex of \(K\) comes from \(\pm E\).
Express each \(v_i\) as a convex combination of signed exceptional roots
and apply Jensen's inequality to \((v_i^\top h)^2\), then average with
\(\lambda_i\).  This proves \eqref{eq:five-comparison} for every \(h\).
Here \(C\) is a second moment; its use requires no mean-zero property of
\(\mu\).
\end{proof}

\begin{lemma}[Noncentered interpolation]
\label{lem:five-interpolation}
Suppose \(z_i\) span \(\R^k\), \(\eta_i\geq0\), and
\(B=\sum_i\eta_i z_iz_i^\top\succ0\).  Set
\(b=\sum_i\eta_i z_i\) and \(a=\sum_i\eta_i\).
Some \(k\) independent rows have an interpolant \(t\) with
\(t^\top Bt\leq ka\).
\end{lemma}
\begin{proof}
For an independent \(k\)-subset \(J\), let \(Z_J\) have rows \(z_j^\top\)
and sample it with probability
\(\eta_J\det(Z_J)^2/\det B\), where \(\eta_J=\prod_{j\in J}\eta_j\).
Cauchy--Binet makes these probabilities sum to one.  Translate all rows by
\(\tau h\).  For invertible \(Z_J\), the coefficient of \(\tau^2\) in its
squared determinant is \(\det(Z_J)^2(h^\top t_J)^2\), where
\(Z_Jt_J=\boldsymbol1\).  Singular subsets contribute nonnegative
coefficients: their determinants are affine in \(\tau\) with zero constant
term.  On the other side of Cauchy--Binet, the coefficient is
\[
 \det B\bigl[(h^\top B^{-1}b)^2+
 (a-b^\top B^{-1}b)h^\top B^{-1}h\bigr].
\]
Consequently, with \(m=B^{-1}b\),
\[
 \mathbb E[t_Jt_J^\top]\preceq
 mm^\top+(a-b^\top B^{-1}b)B^{-1},\qquad
 \mathbb E[t_J^\top Bt_J]\leq ka-(k-1)b^\top B^{-1}b\leq ka.
\]
\end{proof}

Apply this lemma to the exceptional roots projected onto \(P^\perp\),
where \(P\) is a rank-\(r\) circuit span.  Their second moment dominates
the identity on \(P^\perp\).  Interpolate the selected quotient basis and
give the circuit evaluation zero.  Lemma~\ref{lem:five-realization} gives
\begin{equation}
 \delta_6\leq(5-r)\mu(E\setminus P).             \label{eq:five-anchor}
\end{equation}
If two circuit spans cover \(E\), their outside masses \(\alpha,\beta\)
satisfy \(\alpha+\beta\leq1\).  Thus
\begin{equation}
\begin{split}
 (r_A,r_B)=(3,3)&\quad\Longrightarrow\quad
 \delta_6\leq\min(2\alpha,2\beta)\leq1,\\
 (r_A,r_B)=(3,2)&\quad\Longrightarrow\quad
 \delta_6\leq\min(2\alpha,3\beta)\leq6/5.
\end{split}                                                        \label{eq:five-cover}
\end{equation}
Another useful case is a rank-three span \(P\) for which all nonzero
exceptional projections lie on two independent quotient lines.  On each
line choose a root of maximum absolute projected coefficient and interpolate
these two roots with \(z\perp P\).  Every exceptional evaluation has
absolute value at most one, hence so does every original evaluation by
\eqref{eq:five-comparison}.  This gives \(\delta_6\leq1\).

A rank-five circuit gives a six-row fit at zero.  For a rank-four circuit
span \(P\), choose a unit normal \(n\) and a root maximizing
\(|v_i^\top n|\).  The five-row circuit and this root give
\[
 \|z\|^2=\frac{1}{\max_i|v_i^\top n|^2}
 \leq\sum_{v_i\notin P}\lambda_i<1.
\]
Henceforth the geometric classification assumes that all positive circuits
have rank at most three.

\subsection{Signed polar faces and configurations of circuit spans}
\label{app:five-polar}
Work in the second alternative of Lemma~\ref{lem:five-exceptional} and put
\(Q=\{z:|e^\top z|\leq1\ (e\in E)\}\).  This is a bounded, full-dimensional
polytope.  Suppose \(\R^5=F\oplus B\), with dimensions two and three,
and \(u,v\in B\) are independent.  Assume that
\begin{equation}
 P_0=F+\R u,\qquad P_1=F+\R v,\qquad B            \label{eq:five-three-spans}
\end{equation}
are spans of actual rank-three positive circuits.  Let \(H=\spann(u,v)\).

\begin{lemma}[Active-space rigidity]
\label{lem:five-active}
At a vertex \(z\) of \(Q\), either \(z\) or \(-z\) is realizable by
Lemma~\ref{lem:five-realization} using one of these circuits, or all active
roots in \(B\setminus H\) lie on one unoriented line.
\end{lemma}
\begin{proof}
Choose five independent active roots and color them by evaluation \(+1\)
or \(-1\).  If both signs fail realization, each color has quotient rank
at most one for all three spans.  Four independent roots of one color would
span a four-space containing all three circuit spans, contradicting
\(P_0+B=\R^5\); five of one color have quotient rank two directly.
Thus the two color spans have dimensions two and three, say \(U,V\), and
\(\R^5=U\oplus V\).  For each circuit span \(P\), dimension counting gives
\[
 P=(P\cap U)\oplus(P\cap V),\qquad
 \dim(P\cap U)=1,\quad\dim(P\cap V)=2.
\]
Denote these lines by \(L_0,L_1,L_B\), and planes by \(H_0,H_1,H_B\).
The intersections \(P_0\cap B=\R u\), \(P_1\cap B=\R v\) imply
\(H_B=H\), \(L_0\ne L_B\), and \(L_1\ne L_B\).
Since \(\dim(P_0\cap P_1)=2\), either \(L_0=L_1\) or \(H_0=H_1\).
The latter would make their intersections with \(H_B\) both \(\R u\)
and \(\R v\).  Hence \(L_0=L_1\) and
\(H_0\cap H_1\cap H_B=0\).

Every additional active root of the smaller color belongs to
\((P_0+U)\cap(P_1+U)\cap(B+U)=U\); otherwise a quotient color rank
would increase to two.  Likewise every root of the other color belongs
to \((P_0+V)\cap(P_1+V)\cap(B+V)=V\).
Now \(B\cap V=H\) and \(B\cap U=L_B\), proving the assertion.
\end{proof}

If any point of \(Q\) has two nonparallel active roots in \(B\setminus H\),
extend their common face to a vertex.  They remain active, so the lemma
gives a realizable vertex or its negative and \(\delta_6\leq1\).

\begin{lemma}[Three-span cover and four-span configuration]
\label{lem:five-multiple}
Under \eqref{eq:five-three-spans}, if
\(E\subset(F+H)\cup B\), then \(\delta_6\leq1\).
The same conclusion holds, without a cover condition, if the circuit spans
include \(B\) and \(F+\R u_i\) for a basis \(u_0,u_1,u_2\) of \(B\).
\end{lemma}
\begin{proof}
For the first assertion, suppose that every vertex and its negative fail
realization using the three circuits.  The preceding face argument excludes
two nonparallel active roots in \(B\setminus H\) at every point of \(Q\).
Restrict to \(z\perp F\).  Coordinates \((\omega,s)\in H^*\times\R\)
describe this three-dimensional slice.  The roots in \(E\cap(F+H)\)
impose strips on \(\omega\), defining a set \(R\).  Each root in
\(E\cap(B\setminus H)\) imposes an interval in \(s\).
Merge parallel roots by retaining the largest absolute scale.  The slice is
\[
 \omega\in R,\qquad \ell(\omega)\leq s\leq r(\omega),
\]
where \(\ell,r\) are respectively the maximum and minimum of finitely many
affine functions.  The feasible projection \(\Omega\) is compact, convex,
and contains zero in its interior.

Its interval width is positive everywhere: zero width would require
two different nonparallel roots active at opposite endpoints, since one
strip has positive width.  Also \(\Omega=R\), for a segment in \(R\)
from zero to an infeasible point would first encounter zero width.
The active direction at the upper endpoint is constant on the connected
set \(R\); a change forces two nonparallel constraints to tie.  The lower
endpoint has the same property.  At \(\omega=0\), both endpoints belong
to a single root \(w\in B\setminus H\) of greatest absolute transverse
coefficient.  Thus throughout \(R\) the interval is precisely
\(|w^\top z|\leq1\).

The set \(R\) is a bounded two-dimensional strip intersection.  At a
boundary point choose an active root \(g\in E\cap(F+H)\) and orient the
point so that \(g^\top z=1\).  Write
\(g=f+\alpha u+\beta v\), where \((\alpha,\beta)\ne(0,0)\).
Choose the transverse coordinate so that \(w^\top z=1\).  If \(\beta\ne0\),
the roots \(g,w\) are independent modulo \(P_0\); otherwise they are
independent modulo \(P_1\).  Lemma~\ref{lem:five-realization} gives the
desired contradiction.  This argument includes tied and degenerate faces.

For the second assertion, at a vertex failing realization with both signs,
the same two-color decomposition splits each \(F+\R u_i\) into a line
in \(U\) and a plane in \(V\).  Applying the intersection argument with
\(B\) puts all three \(u_i\) in the two-plane \(B\cap V\), contradicting
their independence.
\end{proof}
\subsection{All extensions of a shared-vertex double tetrahedron}
\label{app:five-shared}
Suppose the root set contains two rank-three positive circuits sharing a
root and jointly spanning five dimensions.  Positive rescaling of individual
roots and an invertible coordinate change put their relations in the form
\begin{equation}
 u+a_0+a_1+a_2=0,\qquad u+b_0+b_1+b_2=0,          \label{eq:five-T}
\end{equation}
where \(u,a_0,a_1,b_0,b_1\) are independent.  Write
\(P_A=\spann(u,a_0,a_1)\), \(P_B=\spann(u,b_0,b_1)\).
These transformations will be used only for positivity, ranks, and spans.
Every risk and polar-face argument uses the original roots in
\eqref{eq:five-normal}.

\paragraph{Classification of a crossing exceptional root.}
Every \(g\in E\setminus(P_A\cup P_B)\) has one of the forms
\begin{equation}
\begin{split}
 C_{ij}(\alpha,\beta)&=-(\alpha+\beta)u-\alpha a_i-\beta b_j,\\
 A_{ij}(\alpha,\beta)&=\alpha(u+a_i)-\beta b_j,\\
 B_{ij}(\alpha,\beta)&=-\alpha a_i+\beta(u+b_j),
 \qquad \alpha,\beta>0.                         \label{eq:five-types}
\end{split}
\end{equation}
To see this, write the unique canonical representation
\[
 -g=\gamma u+\sum_i\alpha_i a_i+\sum_j\beta_jb_j,
 \qquad \min_i\alpha_i=\min_j\beta_j=0.
\]
If \(\gamma<0\), replace \(\gamma u\) using the first relation in
\eqref{eq:five-T}.  This represents \(-g\) by at least four independent
core roots with positive coefficients, giving a rank-four or rank-five
positive circuit.  Thus \(\gamma\geq0\).  The positive support has at most
three independent roots for the same reason.  Each leaf block is used, so
their support sizes are \((1,1),(2,1),(1,2)\).
For \((1,1)\), write \(g=-\gamma u-\alpha a_i-\beta b_j\).
If \(\gamma<\alpha+\beta\), the root \(g\) is a positive combination of
\(u\) and the two complementary leaves from each block.  If
\(\gamma>\alpha+\beta\), it is a positive combination of the two
complementary A leaves and all three B leaves.  Both sets are independent;
exceptionality forces equality and gives type \(C\).
For \((2,1)\), \(\gamma=0\).  Unequal A coefficients give a positive
five-root representation of \(g\): add the larger coefficient times the
A zero relation and the B coefficient times the B zero relation.  The
generators are \(u\), two A leaves, and two B leaves.  Equal coefficients
give type \(A\).  The other case gives type \(B\).

\paragraph{Compatibility of crossing roots.}
All claims below follow from the displayed positive relations; independence
can be read in the basis \(u,a_0,a_1,b_0,b_1\).
If \(g=A_{ij}(\alpha,\beta)\), \(h=A_{kl}(\gamma,\delta)\), and \(i\ne k\),
let \(m\) be the third A index.  Then
\begin{equation}
 \gamma g+\alpha h+\alpha\gamma a_m+
 \sum_{r=0}^2(\alpha\gamma+\gamma\beta\boldsymbol1_{r=j}
                  +\alpha\delta\boldsymbol1_{r=l})b_r=0.           \label{eq:five-AA}
\end{equation}
The three B leaves are independent and the A quotient components of
\(g,h\) are independent.  Thus this is a forbidden rank-five circuit.
All exceptional A roots consequently have the same A index; the B claim
is symmetric.

Next let \(g=A_{ij}(\alpha,\beta)\), \(h=B_{kl}(\gamma,\delta)\).
If \(i=k\),
\begin{equation}
 \gamma g+\alpha h+
 \sum_{r=0}^2[\alpha(\gamma+\delta)+\gamma\beta\boldsymbol1_{r=j}
                    -\alpha\delta\boldsymbol1_{r=l}]b_r=0          \label{eq:five-ABsame}
\end{equation}
is a rank-four circuit with strictly positive coefficients.  The case
\(j=l\) is symmetric.  If \(i\ne k\) and \(j\ne l\), with \(m\) the third
B index, then
\begin{equation}
 g=\alpha u+\alpha a_i+\frac{\beta\gamma}{\delta}a_k
       +\beta b_m+\frac\beta\delta h                              \label{eq:five-ABcross}
\end{equation}
makes \(g\) good.  Thus exceptional A and B roots never coexist.

For \(g=C_{ij}(\alpha,\beta)\), \(h=A_{kl}(\gamma,\delta)\), if \(k\ne i\)
and \(m\) is the third A index, then
\begin{equation}
 g=\frac\alpha\gamma h+\alpha a_m+
 \sum_{r=0}^2\left(\alpha+\beta+
 \frac{\alpha\delta}{\gamma}\boldsymbol1_{r=l}-\beta\boldsymbol1_{r=j}\right)b_r
                                                                    \label{eq:five-CAcross}
\end{equation}
is a positive independent five-root representation.  If \(k=i\), \(l\ne j\),
then
\begin{equation}
 \gamma g+\alpha h+\gamma\beta u+\gamma\beta b_j+\alpha\delta b_l=0    \label{eq:five-CAsame}
\end{equation}
is a rank-four circuit.  A coexisting exceptional C root must therefore
use the same pair of indices as every exceptional A root.  Again the B
version is symmetric.

If exceptional A roots occur, fix their common index \(i\) and let
\(F=\spann\{a_k:k\ne i\}\).  Then \(F\oplus P_B=\R^5\).
Each root \(A_{ij}\), together with the two complementary A leaves and
\(b_j\), is a rank-three circuit spanning \(F+\R b_j\).
The original A circuit spans \(F+\R u\).
If two different B indices \(j,l\) occur, the three roots \(u,b_j,b_l\)
form a basis of \(P_B\); the four-span case of
Lemma~\ref{lem:five-multiple} applies.
If only \(j\) occurs, the compatibility relations show that
\(E\subset (F+\spann(u,b_j))\cup P_B\).  The three-span case applies.
Thus this branch, and the symmetric B branch, have \(\delta_6\leq1\).

\paragraph{Crossing roots of type C only.}
Represent each occurring pair \((i,j)\) as an edge in a three-by-three
bipartite graph.  If there are no edges, \(P_A,P_B\) cover \(E\) and
\eqref{eq:five-cover} applies.  If the edges form a star, choose any edge
and anchor its circuit, of span \(\spann(u,a_i,b_j)\).  Both original
blocks project to independent lines, and every crossing exceptional root
projects onto one of them.  The two-line argument gives \(\delta_6\leq1\).

Otherwise there are two disjoint edges.  Relabel and positively rescale their
exceptional witnesses as
\begin{equation}
 h=-(1+p)u-a-pb,\qquad k=-(1+q)u-c-qd,
 \qquad p,q>0,                                                    \label{eq:five-matching}
\end{equation}
where \(a=a_0,c=a_1,b=b_0,d=b_1\),
\(a_2=-u-a-c\), \(b_2=-u-b-d\).
We first prove
\begin{equation}
 E\cap P_A\subset\spann(u,a)\cup\spann(u,c),                       \label{eq:five-blockrestriction}
\end{equation}
and the corresponding statement for \(P_B\).
Write \(g=s u+t a+z c\).  It suffices to exclude \(tz\ne0\).

If \(t,z>0\), put \(R=(1+p)t+(1+q)z-s\).  For \(R\geq0\),
\[
 -g=th+zk+ptb+qzd+Ru.
\]
For \(R>0\) the five parents are independent, and for \(R=0\) the four
remaining parents are independent; both give a high-rank circuit.
For \(R<0\), put \(T=[s-(1+q)z]/(1+p)>t\); then
\[
 -g=Th+zk+(T-t)a+pTb+qzd
\]
has five positive independent parents.
If \(t>0,z<0\) and \(s<(1+p)t\), use
\[
 -g=th+ptb-zc+((1+p)t-s)u.
\]
Otherwise \(K_0=(s-t)/(1+q)>0\) and
\[
 -g=t a_2+(t+K_0-z)c+qK_0d+K_0k.
\]
Both are positive independent four-parent representations.
The other mixed-sign case follows by interchanging \((t,a,h,p,b)\) and
\((z,c,k,q,d)\).
Finally suppose \(t,z<0\), and put \(R'=s-(1+p)t-(1+q)z\).
For \(R'>0\),
\[
 g=-th-zk-ptb-qzd+R'u
\]
makes \(g\) good.  For \(R'<0\), take
\(T=[(1+q)z-s]/(1+p)>-t\) and use
\[
 g=Th-zk+(t+T)a+pTb-qzd.
\]
For \(R'=0\), the independent five-parent identity
\[
 h=-\frac1t g+\left(\frac{s}{t}-1\right)u+\frac zt c+p d+p b_2
\]
makes the exceptional witness \(h\) good, since
\(s/t=(1+p)+(1+q)z/t>1\).  This proves
\eqref{eq:five-blockrestriction}, including every equality boundary.

The other possible C edges are restricted as follows.  For
\(g=C_{01}(x,y)\), \(x,y>0\),
\begin{equation}
 g=xh+\frac yq k+\frac yq c+xp b+\left(xp+\frac yq\right)u          \label{eq:five-C01}
\end{equation}
is a positive independent representation, so this edge is excluded.
For \(g=C_{02}(x,y)\),
\begin{equation}
\begin{split}
 0={}&g+\frac yp h+\frac yq k+\left(x+\frac yp\right)a+\frac yq c\\
    &+\left[x+\frac{y(1+p)}p+\frac{y(1+q)}q\right]u               \label{eq:five-C02}
\end{split}
\end{equation}
is a rank-five circuit.  Swapping the two blocks or the two used indices
excludes all off-diagonal edges.  The only additional edge is \((2,2)\).
For it, \(g=x(a+c)+y(b+d)\).  Set
\[
 r_1=x-y/p,\quad r_2=x-y/q,\quad m=\max(r_1,r_2,0),\quad
 S=y(1+p)/p+y(1+q)/q>0.
\]
Then
\begin{equation}
 g+\frac yp h+\frac yq k+(m-r_1)a+(m-r_2)c+m a_2+(S+m)u=0.          \label{eq:five-C22}
\end{equation}
At least one of the three A coefficients is zero.  Unless \(r_1=r_2=0\),
exactly one or two are positive, yielding a circuit of rank four or five.
For completeness, quotienting by \(u\) fixes the relation coefficients
of \(h,k,g\), from their B coordinates, in the ratio \(y/p:y/q:1\).
The selected A quotient roots are independent, so the remaining
coefficients are uniquely determined.  The coefficient \(S+m\) on \(u\)
is positive, proving that the relation has nullity one.  Therefore
\begin{equation}
 p=q=y/x.                                                        \label{eq:five-equalratio}
\end{equation}

With just the first two diagonal edges, \eqref{eq:five-blockrestriction}
gives the cover
\(E\subset\spann(u,a,b)\cup\spann(u,c,d)\) by two rank-three circuits.
If the third edge occurs, apply the block restriction to all three pairs
of disjoint edges.  Block-internal exceptional roots then lie on \(\R u\).
Equation~\eqref{eq:five-equalratio}, applied to any choice of witnesses,
fixes a common ratio for every diagonal edge.  Their directions are
\[
 h_i=-(1+p)u-a_i-pb_i,\qquad
 \sum_{i=0}^2h_i=-2(1+p)u.
\]
Thus \(E\) has span dimension at most three, contradicting
\eqref{eq:five-comparison}.  We have proved, with arbitrary additional
roots, that a shared-vertex double tetrahedron implies
\begin{equation}
 \delta_6(V)\leq1.                                               \label{eq:five-Tbound}
\end{equation}

\subsection{Systems whose positive circuits have rank at most two}
\label{app:five-lowrank}
Assume now that every positive circuit has rank at most two.  A
\emph{circuit plane} is the span of an actual positive triangle, that is,
a rank-two positive circuit.  We use the following elementary observation
repeatedly.  A support-minimal nonnegative representation of a nonzero
vector has independent positive support: otherwise a dependence permits
a coefficient perturbation until one coefficient vanishes.  If two such
representations give opposite directions in subspaces intersecting in a
line, their union has exactly one dependence and is a positive circuit.

\begin{lemma}[Intersecting circuit planes]
\label{lem:five-planes}
The following statements hold.
\begin{enumerate}
\item Two distinct circuit planes with a nonzero intersection share the
same oriented vertex direction in any chosen witnessing triangles.
\item A family of distinct circuit planes in a three-space, with no
common line, contains at most three planes.
\item Distinct circuit planes sharing a line have independent directions
in the quotient by that line.
\item A circuit plane has zero intersection with at most two members of
a family of distinct circuit planes sharing a line.
\end{enumerate}
\end{lemma}
\begin{proof}
For the first assertion, take \(z\ne0\) on the intersection line and
support-minimal positive representations of \(z,-z\) in each triangle.
Within each triangle their support sizes sum to at least three, since a
positive triangle contains no opposite parallel pair.  The two cross-paired
representations each form a circuit, with at most three roots.  Adding the
two inequalities forces both within-triangle sums to equal three.
Each triangle therefore has a singleton representation of one direction.
These directions must agree; if they opposed one another, the other
cross-pair would contain four roots.

For the second assertion, choose three planes with no common line.  Their
pairwise intersection directions are independent; by the first assertion,
the witnessing triangles can be written, after positive scaling, as
\[
 \{u,v,-au-bv\},\quad \{u,w,-cu-dw\},\quad
 \{v,w,-ev-fw\},\qquad a,b,c,d,e,f>0.
\]
A fourth plane containing none of the axes would have to contain all three
negative edge directions.  Their determinant has absolute value
\(ade+bcf>0\), which excludes a plane.
If the fourth plane contains \(u\), it contains neither \(v\) nor \(w\).
Its intersection with the third plane must be the direction \(-ev-fw\),
so its third triangle root has the form
\(g=-\alpha u+\beta ev+\beta fw\), \(\alpha,\beta>0\).  But
\[
 g+\frac{\beta e}{b}(-au-bv)+\frac{\beta f}{d}(-cu-dw)
 +\left(\alpha+\frac{\beta ea}{b}+\frac{\beta fc}{d}\right)u=0
\]
is a rank-three positive circuit.  The other axes are symmetric.

For the third assertion, the first part gives a common oriented root
\(u\).  Normalize each triangle to
\(u+x_i+y_i=0\), and write
\(x_i=e_i-u/2\), \(y_i=-e_i-u/2\).
A minimal dependence of their quotient directions would be
\(\sum_{i=1}^kt_i e_i=c u\), with \(t_i\ne0\), \(k\geq3\).
Put \(S=\sum_i|t_i|\).  The two choices of signed leaf directions give
positive combinations equal to \((c-S/2)u\) and \((-c-S/2)u\).
At least one coefficient is negative.  Adding a positive multiple of
\(u\) gives a rank-\(k\) circuit, by minimality of the quotient dependence,
a contradiction.

For the last assertion, take three shared-line planes
\(Q_i=\spann(u,e_i)\); their sum \(W\) has dimension four.  Their
seven-root core consists of \(u\) and \(\pm e_i-u/2\).
Suppose a circuit plane \(P\) is disjoint from every \(Q_i\).
If \(P\not\subset W\), then \(P\cap W\) is a line containing no core
root direction.  Both signs of that direction need at least two core
roots, and at least one sign needs two roots in the triangle spanning
\(P\).  Cross-pairing gives a circuit of rank at least three.
If \(P\subset W\), write a triangle root as
\(p=z_0u+\sum_{i=1}^3z_i e_i\).  It has at least two nonzero leaf
coordinates, since \(P\cap Q_i=0\).  A minimal core representation of
\(-p\) must have exactly two parents to avoid a higher-rank circuit.
They are leaves of different arms, forcing
\(z_0=(|z_i|+|z_j|)/2>0\).
This holds for all three roots of the triangle spanning \(P\),
contradicting their positive zero relation.
\end{proof}

Let \(G\) be the finite simple graph whose vertices are the distinct
circuit planes, with an edge exactly when two planes have zero
intersection.
\begin{lemma}[Short cycles give a four-plus-one decomposition]
\label{lem:five-shortcycles}
If \(G\) has a triangle or a four-cycle, every root lies in \(W\cup L\)
for a four-space \(W\) and a complementary line \(L\).  In this case
\(\delta_6\leq1\).
\end{lemma}
\begin{proof}
For a graph triangle, the first two planes span a four-space \(W\).
If the third plane lay outside it, their intersection line would contain
no root of the first two witnessing triangles, by pairwise disjointness.
Its two signs need at least two roots from that core, and at least one
sign needs two from the third triangle.  Their cross-pair gives a circuit
of rank at least three.  Thus all three planes lie in \(W\).

A chorded four-cycle contains a triangle.  For an induced four-cycle,
write its parts as \(A_1,A_2\) and \(B_1,B_2\), and put
\(U=A_1+A_2\), \(V=B_1+B_2\).  These are three-spaces.  If their sum
had dimension five, their intersection would be a line.  Each side has
a positive generating core formed by two triangles with a same-oriented
shared vertex, and hence no opposite parallel roots.  For a direction
on the intersection line, the minimal representation sizes of its two
signs sum to at least three on either side.  Cross-pairing produces two
circuits, each with at most three roots.  Equality is forced in both
sums, so each side has a singleton representation.  The intersection
line would then lie in some \(A_i\) and some \(B_j\), contradicting a
cross-edge.  The four planes therefore lie in a four-space \(W\).

Consider any other circuit plane \(D\not\subset W\).
In the triangle case it meets at most one original plane nontrivially;
with the other two it gives a triangle spanning five dimensions.
In the four-cycle case, if it meets at most one original plane,
choose an adjacent pair among the other three and obtain the same
contradiction.  If it meets two, they must belong to one part of the
cycle, and it meets neither plane of the other part.  Replacing one plane
of the first part by \(D\) gives an induced four-cycle spanning five
dimensions.  Thus every circuit plane is in \(W\).

Every root outside \(W\) must now belong to a rank-one circuit and
therefore has a root on its opposite ray.  If two distinct exterior
lines occurred, choose roots on them with opposite quotient directions
modulo \(W\).  Their positive combination gives a nonzero \(z\in W\).
A minimal positive representation of \(-z\) in the core of \(W\),
combined with these two exterior roots, gives a circuit of rank at least
two containing an exterior root.  This contradicts the preceding result.
Thus all exterior roots lie on one complementary line \(L\).

Let the evaluation masses in \(W,L\) be \(a,b>0\), \(a+b=1\).
Mean zero splits along this direct sum.  Their two positive semidefinite
second moments have ranks four and one and sum to \(I_5\); their
complementary ranges are therefore orthogonal, and the moments are the
orthogonal projections.  In the conditional four-dimensional system,
scaling the roots by \(\sqrt a\) gives unit covariance.  The
four-dimensional five-row theorem of \cite[Theorem~47]{zhang2026exact}
then gives an excess infimum at most \(a\) in the original coordinates.
On \(L\), interpolate a root of greatest absolute coefficient to give
excess at most \(b\).  Combining the two full-rank fits uses at most six
rows and gives excess at most \(a+b=1\).
The positive-weight form of the four-dimensional input is justified in
Subsection~\ref{app:five-transfer} below.
\end{proof}

In the remaining case \(G\) has no triangles or four-cycles.  Neighbors
of a vertex are pairwise intersecting planes.  Such a plane family either
shares a line or lies in a three-space.  Indeed, two planes span a
three-space \(S\) and meet in a line \(\ell\).  A third meeting both
without containing \(\ell\) lies in \(S\).  A plane outside \(S\) meeting
the first two would have to contain \(\ell\), and could then not meet
that third plane.  This proves the dichotomy.
Lemma~\ref{lem:five-planes} bounds the first case by two neighbors and
the second by three.  Consequently \(\Delta(G)\leq3\).

\begin{lemma}[A graph probability]
\label{lem:five-graph}
Every nonempty finite graph with no triangles or four-cycles and maximum
degree at most three has a vertex probability \(\theta\) such that
\begin{equation}
 \sum_{A\sim B}\theta_A-\theta_B\leq\frac15
 \quad\hbox{for every vertex }B.                                \label{eq:five-graphprob}
\end{equation}
\end{lemma}
\begin{proof}
If a vertex has degree three, assign it mass \(2/5\) and its three
neighbors mass \(1/5\) each.  The left side is \(1/5\) on this star.
An exterior vertex meets at most one positive-mass neighbor, since two
would create a four-cycle, and it does not meet the center.
If the maximum degree is at most two, choose one component.
On a cycle of length \(m\geq5\), use the uniform probability.
For paths on one, two, three, and four vertices, respectively use
\[
 (1),\quad (1/2,1/2),\quad (2/7,3/7,2/7),\quad (1/6,1/3,1/3,1/6).
\]
Their left sides have maxima \(-1,0,1/7,1/6\), respectively.
For a longer path put mass \(1/5\) on each of five consecutive vertices.
Every vertex, including those immediately outside the segment, satisfies
\eqref{eq:five-graphprob}.  All other components have zero mass.
\end{proof}

Choose one actual triangle for each graph vertex as an anchor.
For any rank-two circuit \(B\), set
\[
 \rho_{AB}=2-\dim(P_A\cap P_B)
   =1+\boldsymbol1_{A\sim B}-\boldsymbol1_{P_A=P_B}.
\]
For a rank-one circuit, set \(\rho_{AB}=1-\dim(P_A\cap P_B)\leq1\).
Equation~\eqref{eq:five-graphprob} gives
\(\sum_A\theta_A\rho_{AB}\leq6/5\) for every rank-two circuit.

We now convert this inequality to the original evaluation risk.
For a fixed anchor \(A\), the polytope
\[
 \mathcal P_A=\{z\in P_A^\perp:v_i^\top z\leq1\text{ for all }i\}
\]
is bounded, three-dimensional, and contains zero in its interior.
Express zero as a convex combination of its vertices.  Each vertex
has three active roots independent modulo \(P_A\).  These interpolating
roots and the zero-evaluation triangle anchor produce an actual
positive-weight full-rank six-row fit.

Fix a positive circuit \(B\), with normalized coefficients
\(\beta_i>0\), \(\sum_{i\in B}\beta_i=1\).
Project it modulo \(P_A\), discard zero projections, and let their
remaining coefficient mass be \(m_{AB}\leq1\).  Decompose the resulting
positive zero relation into normalized projected circuits.  Their
mixing masses sum to \(m_{AB}\), and their ranks are at most \(\rho_{AB}\).
For one such circuit with coefficients \(\alpha_i\) and support size
\(s\), set \(T_i=v_i^\top z\) under the mean-zero vertex distribution.
Then \(T_i\leq1\), \(\sum_i\alpha_iT_i=0\) pointwise, and
\(\mathbb ET_i=0\).  Hence
\[
 -(1-\alpha_i)/\alpha_i\leq T_i\leq1,\qquad
 \mathbb ET_i^2\leq(1-\alpha_i)/\alpha_i.
\]
The second inequality follows by averaging
\((T_i+(1-\alpha_i)/\alpha_i)(T_i-1)\leq0\).
Thus \(\mathbb E\sum_i\alpha_iT_i^2\leq s-1\leq\rho_{AB}\).
Reassemble the projected circuits to obtain
\begin{equation}
 \mathbb E_A\sum_{i\in B}\beta_i(v_i^\top z)^2
       \leq m_{AB}\rho_{AB}\leq\rho_{AB}.                         \label{eq:five-realrisk}
\end{equation}
Finally mix anchors according to \(\theta\), and decompose the original
mean-zero probability \(\lambda\) into normalized positive circuits.
Use \eqref{eq:five-realrisk}, the graph bound for rank-two circuits,
and the bound one for rank-one circuits.  The original covariance gives
\begin{equation}
 \mathbb E\|z\|^2
   =\mathbb E\sum_i\lambda_i(v_i^\top z)^2\leq\frac65.             \label{eq:five-lowrisk}
\end{equation}
These are finite distributions on actual fits, so some fit satisfies
the same bound.

If there are no circuit planes, every root has an opposite root, and
their distinct unoriented lines are independent.  Otherwise a minimal
line dependence, oriented using the available opposite roots, would
give a positive circuit of rank at least two.  There are therefore five
independent lines; choose a maximum-absolute-coefficient root on each
and interpolate.  All evaluations have absolute value at most one.
This also gives \(\delta_6\leq1\), and completes the low-rank case.

\subsection{The remaining rank-three case and the global root bound}
\label{app:five-remaining}
Suppose that a rank-three circuit \(A\) exists, that every circuit has
rank at most three, and that the configuration \eqref{eq:five-T} is absent.
First suppose that no seven roots positively generate \(\R^5\).
In the two-dimensional quotient by \(P_A\), a positive generating
triangle would lift, together with the four anchor roots, to a
seven-root positive generating set.  A planar positive generating system
with no such triangle lies on two independent lines: a minimal positive
basis consists of two opposite pairs, and any root off their axes would
form a positive triangle with two suitably oriented axis roots.
Choosing maximum-absolute-coefficient roots on these quotient lines,
and interpolating with zero anchor evaluations, gives \(\delta_6\leq1\).

Otherwise choose a minimal positive generating subset of seven roots.
It has exactly seven roots, since a six-root positive generating set
in dimension five would be a rank-five circuit.  Its relation space
has dimension two, and its nonnegative relation cone is full-dimensional
and pointed in that space, with two extreme rays.  Their circuit
supports cover the seven roots and each has at most four roots.
The possibilities are disjoint supports of sizes four and three, or
two four-root supports sharing one root.  The latter is
\eqref{eq:five-T}, already excluded.  The former gives complementary
spaces \(U,W\) of dimensions three and two, with normalized relations
\begin{equation}
 a_0+a_1+a_2+a_3=0,\qquad b_0+b_1+b_2=0,
 \qquad U\oplus W=\R^5.                                         \label{eq:five-complement}
\end{equation}

For a crossing exceptional root \(g\), write \(-g\) as a nonnegative
combination in each simplex with minimum coefficient zero.  If the
positive support sizes are \(p,q\geq1\), their union is independent;
avoiding a high-rank circuit forces \(p+q\leq3\).
For \((1,1)\), the complementary three A and two B roots positively
generate \(g\), making it good.
For \((1,2)\), unequal B coefficients again make \(g\) good, while equal
coefficients give \(g=-\alpha a_i+\beta b_k\).
Together with \(a_i\) and the other two B roots this is a rank-three
circuit sharing \(a_i\) with the original A circuit and jointly
spanning dimension five, a forbidden shared-vertex configuration.
For \((2,1)\), unequal A coefficients likewise make \(g\) good, so
the remaining form is
\begin{equation}
 g=-\alpha(a_i+a_j)-\beta b_k,\qquad\alpha,\beta>0.                \label{eq:five-cross32}
\end{equation}
It gives a rank-three circuit with span \(\spann(a_i,a_j,b_k)\).

Encode these roots in a bipartite graph.  The three unordered
two-plus-two partitions of the four A indices form one side, and the
three B indices form the other.  Two edges disjoint on both sides
would have A pairs sharing exactly one index.  Their two circuits
would share that real A root and span five dimensions, again giving
\eqref{eq:five-T}.  Thus the occupied edges pairwise intersect and form
a star.  Choose one crossing circuit as anchor.  If the A partition
is fixed, every crossing A component is zero in its quotient, including
the complementary pair because the four A roots sum to zero.
If the B index is fixed, every crossing B component is zero there.
Both original blocks project onto independent lines.  In either case,
all exceptional projections lie on these two lines, giving
\(\delta_6\leq1\).  If there are no crossing exceptional roots,
the \((3,2)\) cover in \eqref{eq:five-cover} gives \(\delta_6\leq6/5\).

The cases now exhaust all root systems: ranks four and five were
handled in Subsection~\ref{app:five-interfaces}, the shared-vertex case
in Subsection~\ref{app:five-shared}, and ranks at most two in
Subsection~\ref{app:five-lowrank}.  Therefore every finite normalized
root system satisfies
\begin{equation}
 \delta_6(V)\leq\frac65.                                         \label{eq:five-globalroot}
\end{equation}

\subsection{Transfer to the original weighted regression problem}
\label{app:five-transfer}
We use the four-dimensional five-point upper bound two from
\cite[Theorem~47]{zhang2026exact}, including its construction by
strictly convex finite training problems or limits of such problems.
Its uniform-evaluation statement extends to strictly positive evaluation
weights as follows.  Approximate them by positive rational probabilities
on the same support, and duplicate rows to obtain uniform datasets.
Duplicate selected rows can be merged without increasing support.
The set of full-rank training outputs is independent of evaluation
weights.  The evaluation covariance has a uniform positive lower
eigenvalue near the original weights, so bounded-risk selected outputs
are bounded.  The changes in their evaluation risks consequently tend
to zero uniformly on this bounded set, as does the change in the
optimal risk.  Taking infima proves the weighted statement.

For observed feature rank four, apply that input after an isometric
restriction.  For rank three use the five-point boundary theorem
already proved in this paper; ranks at most two admit full-span exact
recovery with at most four points.  The zero-optimal-risk case was
settled in the problem formulation.  It remains to consider rank five
and \(L^\star>0\).

Set
\[
 H=\sum_i p_i x_ix_i^\top\succ0,\quad
 \xi_i=H^{-1/2}x_i,\quad
 \eta_i=\frac{y_i-x_i^\top w^\star}{\sqrt{L^\star}},\quad
 z=\frac{H^{1/2}(w-w^\star)}{\sqrt{L^\star}}.
\]
Then
\begin{equation}
 \sum_i p_i\xi_i\xi_i^\top=I_5,\quad
 \sum_i p_i\eta_i\xi_i=0,\quad
 \sum_i p_i\eta_i^2=1,\quad
 \frac{L_D(w)}{L^\star}=1+\|z\|^2.                               \label{eq:five-whitening}
\end{equation}
Every full-rank training problem has a unique solution, so the invertible
change of coordinates respects the original minimum-norm learner.

If some row has \(\eta_0=0\), \(\xi_0\ne0\), project the regression
problem onto \(\xi_0^\perp\).  Its covariance is the identity on this
four-space, its optimum is zero, and its optimal risk is one.
The four-dimensional input provides a strictly convex training objective
\(F\) on at most five projected rows, whose constrained minimizer has
squared norm at most \(1+\epsilon\).
Lift those rows and add row zero with weight \(M>0\), giving
\(F(z)+M(\xi_0^\top z)^2\).  Its Hessian is positive definite since
\(F\) is strictly convex on \(\xi_0^\perp\).  As \(M\to\infty\), its
unique solution converges to the constrained solution.  This follows
also by taking the Schur complement in the normal equations relative
to \(\xi_0^\perp\) and its normal line.  Each finite \(M\) uses at most
six rows with positive weights.  Letting \(M\to\infty\) and then
\(\epsilon\downarrow0\) gives risk-ratio infimum at most two.

We may therefore assume \(\eta_i=0\) implies \(\xi_i=0\).
For \(\eta_i\ne0\), put \(v_i=\xi_i/\eta_i\) and
\(\lambda_i=p_i\eta_i^2\).  Remove the zero roots, corresponding to
pure-residual rows, and write \(a\in(0,1]\) for the mass of the remaining
roots.  Their mean is zero and their second moment is \(I_5\).
The rescaled system
\(\widetilde v_i=\sqrt a\,v_i\), \(\widetilde\lambda_i=\lambda_i/a\)
satisfies \eqref{eq:five-normal}.  Equation~\eqref{eq:five-globalroot}
gives a sequence of six-row full-rank positive-weight outputs, and
\(z=\sqrt a\,\widetilde z\) gives
\[
 \inf\|z\|^2\leq a\frac65\leq\frac65.
\]
Convert each root training weight \(w_i\) to the original sample weight
\(w_i/\eta_i^2\), and normalize all selected weights by their sum.
The selected objective changes only by a positive common factor;
support, positivity, and full rank are preserved.  Thus
\begin{equation}
 \Gamma_5(6)\leq\frac{11}{5}.                                    \label{eq:five-upper}
\end{equation}

\subsection{A matching seven-point construction, including deficient ranks}
\label{app:five-lower}
Let
\[
 x_0=e_1,\quad x_1=e_2,\quad x_2=e_3,\quad x_3=-e_1-e_2-e_3,
 \qquad x_4=e_4,\quad x_5=e_5,\quad x_6=-e_4-e_5.
\]
Give the first four rows evaluation weight \(1/10\) each and the last
three weight \(1/5\) each.  Set
\[
 h=(10,20,40,80,160)^\top,\qquad y_i=x_i^\top h+1.
\]
The weighted row mean is zero, so the full-data optimum is \(h\) and
its risk is one.  The covariance is
\[
 M=\diag\left(\tfrac1{10}(I_3+\boldsymbol1\boldsymbol1^\top),
              \tfrac15(I_2+\boldsymbol1\boldsymbol1^\top)\right),
\]
and the excess risk is \((w-h)^\top M(w-h)\).

A full-rank support of size at most six must use at least three rows
from the first block and at least two from the second.  At least one
block has exactly its dimension \(r\) rows.  Those independent rows
interpolate under any positive training weights.  With \(z=w-h\),
their evaluations are one, and the omitted simplex row evaluates to
\(-r\).  If the block mass is \(a_r\), its excess risk is
\[
 \frac{a_r}{r+1}(r+r^2)=a_r r=\frac65,
 \qquad a_3=\frac25,\quad a_2=\frac35.
\]
The other block has nonnegative excess.  Equality is attained by taking
all four A rows with equal training weights, giving zero block displacement,
and any two B rows, which interpolate.

For a rank-deficient selection, the minimum-norm solution lies in its
selected feature span.  If the three-dimensional block is deficient,
that span is contained in a plane generated by two of its simplex
rows.  Normals of these six planes are coordinate vectors or
differences of coordinate vectors.  The squared distance of
\(h_A=(10,20,40)\) from every such plane is at least
\[
 \min\{100,400,1600,(20-10)^2/2,(40-10)^2/2,(40-20)^2/2\}=50.
\]
Since the A covariance dominates \(I_3/10\), its excess risk is at
least five.  If the B block is deficient, the squared distance of
\(h_B=(80,160)\) from its three possible root lines is at least
\[
 \min\{80^2,160^2,(160-80)^2/2\}=3200.
\]
Its covariance dominates \(I_2/5\), giving excess at least \(640\).
Lower-dimensional selected spans are contained in these planes or
lines and satisfy the same bounds.  Thus all deficient-rank selections
have excess greater than \(6/5\).  The instance has selected-risk ratio
exactly \(11/5\), proving
\(\Gamma_5(6)\geq11/5\).  Together with \eqref{eq:five-upper}, this
establishes the five-dimensional six-point law.

%% file: appendix/upper_half_verified.tex
This appendix proves Theorem~\ref{thm:upper-half}. Throughout, training
weights are nonnegative and sum to one, and only their positive entries
count toward the support. All limits below preserve the selected feature
span at every positive mixture parameter.

\subsection{Certificates and positive spanning sets}
Set $g_i=(x_i^\top\wstar-y_i)x_i$. The full normal equation is
$\sum_i p_i g_i=0$. A nonempty support $B$ is a full-span certificate if
its features span $U$ and some strictly positive coefficients on $B$
balance the $g_i$. Normalizing those coefficients gives an admissible
selection returning $\wstar$. Indeed, the selected and full minimizer
sets are both $\wstar+U^\perp$. A minimum-cardinality certificate exists
because the complete finite dataset is a certificate.

We give the positive-basis ingredients directly. For a finite set $S$,
write $\operatorname{pos}S$ for its nonnegative linear combinations.
A positive basis of a vector space is an inclusion-minimal finite set
whose positive hull is the whole space.

\begin{lemma}[Positive spanning compression and anchors]
\label{lem:uh-positive}
In a finite-dimensional real vector space of dimension $r$:
\begin{enumerate}[leftmargin=*,itemsep=1pt]
\item every nonnegative combination has a representation on a linearly
independent subset;
\item a positive basis has at most $2r$ elements;
\item if $S$ positively spans the space and $C\subseteq S$ positively
spans its own $a$-dimensional span, then some $T\subseteq S$ positively
spans the whole space and satisfies $|T|\leq |C|+2(r-a)$;
\item a positive basis of size $2r-j$, $j\geq1$, contains a positively
balanced simplex of dimension $a$ with $2\leq a\leq j+1$.
\end{enumerate}
\end{lemma}
\begin{proof}
For the first assertion, a dependence among the currently positive
coefficients permits a coefficient adjustment along a nonzero zero
relation. Orient the relation so that one coefficient is positive,
and subtract the smallest ratio of a current coefficient to a positive
relation coefficient. Nonnegativity and the represented vector are
preserved, and the support strictly decreases. Iteration proves the claim.

Choose a linear basis $B$ from a positive basis $S$. Represent
$-\sum_{b\in B}b$ nonnegatively on an independent subset $T\subseteq S$,
so $|T|\leq r$. Every $-b$ belongs to $\operatorname{pos}(B\cup T)$.
Hence $B\cup T$ positively spans the whole space, and minimality gives
$S=B\cup T$, proving $|S|\leq2r$.

For the third assertion, let $W=\spann C$. The image of $S$ positively
spans the quotient by $W$. Choose a positive basis there, of size at
most $2(r-a)$, and one original representative per quotient vector.
Their nonnegative combinations represent every quotient class. Subtract
such a representative combination from the desired vector; the remainder
lies in $W$ and is a nonnegative combination of $C$. The union has the
required size and positive span. In particular, if $S$ itself is a
positive basis, minimality implies
\begin{equation}
 |S|\leq |C|+2(r-a).
 \label{eq:uh-anchor-card}
\end{equation}

For each $x\in S$, represent $-x$ by a strictly positive combination of
an independent subset $T\subseteq S$. Then $x\notin T$, and
$C=T\cup\{x\}$ is a balanced simplex with dimension $|T|$.
A positive basis contains no zero vector, so $|T|\geq1$.
If every such simplex had dimension one, each generator would have an
opposite positive multiple in $S$. A linear basis and its opposite
partners would positively span the space and would consist of exactly
$2r$ distinct vectors. Minimality would give $|S|=2r$, a contradiction.
Thus some simplex has $a\geq2$. Applying \eqref{eq:uh-anchor-card} with
$|S|=2r-j$ and $|C|=a+1$ gives $a\leq j+1$.
\end{proof}

\begin{lemma}[Full-rank certificate alternatives]
\label{lem:uh-alternatives}
Suppose $U=\R^d$, $d>0$. For a budget $2d-t$, either an exact full-span
selection fits the budget, a minimum certificate has $2d$ active rows,
or there is a balanced anchor $C$ whose feature span has dimension $a$,
with either $(a,|C|)=(1,1)$ or $2\leq a\leq t$, $|C|=a+1$.
\end{lemma}
\begin{proof}
Choose a minimum certificate $B$. No row in $B$ has zero feature:
removing it would preserve balance and span, and the remaining support
is nonempty because the feature rank is positive. If some $g_i=0$ on
$B$, its nonzero feature is a singleton balanced anchor of dimension one.
Otherwise its gradients and features have the same span. Strict positive
balance makes the gradients positively span $\R^d$. They form a positive
basis with distinct vectors: a smaller positively spanning set, or one
representative for each repeated gradient, would give a smaller
full-span certificate. Lemma~\ref{lem:uh-positive} gives $|B|\leq2d$.
If $2d-t<|B|<2d$, write $|B|=2d-j$, $1\leq j<t$, and use its simplex
anchor. Nonzero collinearity of gradients and features preserves the
anchor dimension and its original row count. The other two cardinality
cases give the stated alternatives.
\end{proof}

\subsection{The maximum-cardinality branch uses the complete dataset}
\label{app:uh-maximal}
The following argument includes gradients outside the chosen certificate.
Suppose a minimum full-span certificate has $2d$ active rows. Their
gradients positively span $\R^d$, so the full gradient set $S$ does too.
Every positively spanning subset of $S$ has at least $2d$ elements:
choose one original row per vector, use strict positive balance, and
observe that spanning gradients imply spanning features.

Consequently every balanced simplex $C\subseteq S$ has dimension at
most one. Otherwise Lemma~\ref{lem:uh-positive}(3) gives a positively
spanning set of size at most $(a+1)+2(d-a)<2d$.
Choose a positive basis $B\subseteq S$. For any nonzero $x\in S$, an
independent positive representation of $-x$ using $B$ gives a simplex
through $x$. It must use a single vector, so $x$ has an opposite
positive multiple in $B$. In particular every element of $B$ has a
partner. Choose a linear basis from $B$ and these partners; their union
positively spans the space, so minimality makes it all of $B$.
Every gradient in the complete dataset therefore lies on one of these
$d$ independent lines.

If a nonzero feature anywhere in the dataset has zero gradient, use
it as a singleton anchor in the induction below. In the remaining case,
every nonzero feature lies on the same independent lines. Write
$x_i=\alpha_i v_{j(i)}$, with $v_1,\ldots,v_d$ a linear basis, and
$r_i=x_i^\top\wstar-y_i$. For the nonzero rows on line $j$, put
\[
 H_j=\sum_{i:j(i)=j}p_i\alpha_i^2,\qquad
 R_j=\sum_{i:j(i)=j}p_i r_i^2.
\]
The basis coordinates $w\mapsto(v_j^\top w)_j$ are invertible, even
when the lines are not orthogonal. The full normal equation balances
each line separately. Two original rows suffice to keep its coordinate
unchanged: use one zero-gradient row if present, or balance two
oppositely signed scalar gradients with positive weights.

One row can set the displacement coordinate to $u_j=-r_i/\alpha_i$.
Its smallest excess $\delta_j$ obeys $0\leq\delta_j\leq R_j$, because
the average of $H_j(r_i/\alpha_i)^2$ under probabilities
$p_i\alpha_i^2/H_j$ is exactly $R_j$. Compress the $t$ least costly
lines to one row and keep at most two rows on each other line.
Their union spans $\R^d$, their local normal equations combine, and
normalizing all weights leaves the minimizer unchanged. Thus
\begin{equation}
 |\supp(q)|\leq2d-t,\qquad
 L_D(\widehat w(q))-\LD
 =\sum_{j\text{ compressed}}\delta_j
 \leq\frac td\sum_j\delta_j
 \leq\frac td\sum_jR_j\leq\frac td\LD.
 \label{eq:uh-lines}
\end{equation}
Zero features contribute only a nonnegative constant to the evaluation
loss. Every line retains a nonzero feature, so this construction returns
the required minimum-norm fit.

\subsection{Balanced anchors and an actual positive-weight limit}
\label{app:uh-lift}
Let $U=\R^d$ and let normalized weights $\beta$ on an anchor $C$
balance the gradients at $\wstar$. Write $W=\spann\{x_i:i\in C\}$,
$E=W^\perp$, and let $P,Q$ be the respective orthogonal projections.
The projected instance on $E$ has features $Qx_i$, responses $-r_i$,
and the original evaluation weights $p_i$. Its full optimum is zero,
its features span $E$, and its optimal risk is $\LD$.

Suppose a projected selection $\gamma$ spans $E$ and returns $z\in E$.
Set $A=\sum_i\beta_i x_ix_i^\top$,
$B=\sum_i\gamma_i x_ix_i^\top$, and $b=\sum_i\gamma_i r_ix_i$.
Then $Q(Bz+b)=0$. For $0<\epsilon<1$, use
\[
 q_\epsilon=(1-\epsilon)\beta+\epsilon\gamma.
\]
Weights on overlapping indices are added. The support is the union,
has size at most $|C|+|\supp\gamma|$, and spans $W\oplus E$.
To control its solution, define endomorphisms
\begin{equation}
 M_\epsilon=(1-\epsilon)A+\epsilon PB+QB,
 \qquad f_\epsilon=-(\epsilon P+Q)b.
 \label{eq:uh-limit-system}
\end{equation}
The limiting map $M_0=A+QB$ is invertible. If $M_0u=0$, projection
onto $W$ gives $Au=0$, hence $u\in E$. It also gives $QBu=0$.
Thus $u^\top Bu=u^\top QBu=0$, so $u$ is orthogonal to the secondary
features as well as the anchor features. Their union spans the whole
space, forcing $u=0$.

Since $Az=0$ and $Q(Bz+b)=0$, we have $M_0z=f_0$. Continuity of
matrix inversion gives $u_\epsilon=M_\epsilon^{-1}f_\epsilon\to z$.
Moreover,
\[
 (P+\epsilon Q)M_\epsilon=(1-\epsilon)A+\epsilon B,
 \qquad (P+\epsilon Q)f_\epsilon=-\epsilon b.
\]
These are precisely the selected normal equations for
$\wstar+u_\epsilon$. The selected Gram matrix is positive definite,
so this point is the unique selected minimizer. Invertibility of
$M_\epsilon$ also follows from this identity for every
$0<\epsilon<1$. Finally,
\[
 L_D(\wstar+z)=\sum_i p_i(r_i+x_i^\top z)^2=\bar L(z).
\]
For every $\zeta>0$, continuity therefore gives a legal full-span
selection with risk less than $\bar L(z)+\zeta$.

\subsection{Dimension induction and deficient ranks}
\label{app:uh-induction}
First prove, for every $d,t\geq0$ with $2t\leq d$, the full-rank
statement: for every $\eta>0$, some selection has full span, budget
$\max\{1,2d-t\}$, and risk at most $(1+t/d)\LD+\eta$.
At $d=t=0$ use one sample; all features and the minimum-norm solution
are zero. The convention for this auxiliary case is $t/d=0$.
Proceed by strong induction on $d$.

If any nonzero feature has zero gradient, use it as a singleton anchor.
The projected parameters are $(d-1,(t-1)_+)$, where $(t-1)_+=\max\{t-1,0\}$.
They satisfy the induction condition and
\[
 1+\max\{1,2(d-1)-(t-1)_+\}\leq\max\{1,2d-t\},
 \qquad \frac{(t-1)_+}{d-1}\leq\frac td.
\]
For $d=1$, necessarily $t=0$; the quotient is the already settled
zero-dimensional case and the total budget is two. Use tolerance
$\eta/2$ in the projected problem and $\eta/2$ in the lifting lemma.

Otherwise apply Lemma~\ref{lem:uh-alternatives}. An exact certificate
within budget gives the claim immediately. A $2d$-row minimum
certificate is handled by \eqref{eq:uh-lines}. In the remaining case,
the anchor has dimension $2\leq a\leq t$ and $a+1$ rows. Set
\[
 d'=d-a,\qquad t'=t-a+1.
\]
Then $d'<d$, $1\leq t'<t$, $2t'\leq d'$, and
\[
 (a+1)+(2d'-t')=2d-t,
 \qquad \frac{t'}{d'}\leq\frac td.
\]
For the latter inequality, cross multiplication reduces to
$d(a-1)\geq ta$, which follows from $d\geq2t$ and $a\geq2$.
Apply the induction hypothesis on the complete projected dataset and
lift, again dividing the tolerance equally. The auxiliary case $t=0$
closes as well: without a singleton anchor, a minimum active certificate
has at most $2d$ rows and is exact within budget.

Now let the actual feature rank be $m\leq d$ and put $n=2d-t$.
Rank zero is exact with one selected row. For $m>0$, identify the
feature span with $\R^m$ and set $k=(2m-n)_+$. Direct integer arithmetic
gives
\[
 2k\leq m,\qquad \max\{1,2m-k\}\leq n,
 \qquad \frac{k}{m}\leq\frac td.
\]
The full-rank result in dimension $m$ therefore proves the same bound.
For completeness, the minimum-norm solution belongs to the feature
span: its orthogonal projection has identical predictions and smaller
norm unless the orthogonal component is zero. The selected span equals
the full feature span, so the same argument applies after selection.
This justifies both coordinate reduction and return to the original
minimum-norm convention. We have proved the additive bound for every
$\eta>0$; take $\eta=\varepsilon\LD$ to obtain the upper bound in
Theorem~\ref{thm:upper-half}.

\subsection{Matching lower bound, including omitted coordinates}
Take the uniform $2d$-row dataset $(e_j,1),(e_j,3)$, $j\in[d]$.
Its full solution is $(2,\ldots,2)$ and $\LD=1$. Every output $u$ has
\[
 L_D(u)=1+\frac1d\sum_{j=1}^d(u_j-2)^2.
\]
If a coordinate retains two, one, or zero rows, its least possible
squared displacement is respectively $0$, $1$, or $4$. In the last
case the minimum-norm rule forces $u_j=0$. If $A$ coordinates retain
one row and $B$ retain none, the budget gives $A+2B\geq t$ and hence
$L_D(u)\geq1+(A+4B)/d\geq1+t/d$. Retain one endpoint on exactly $t$
coordinates and both endpoints on the others. Equal weights within
each retained pair, followed by normalization of all weights, attain
$1+t/d$. This proves the matching sharpness and completes the theorem.

\subsection{Formal theorem and dependency map}
\label{app:uh-formal}
The supplied project uses Lean~4.19.0 and mathlib commit
\texttt{c44e0c8ee63ca166450922a373c7409c5d26b00b}.
Its declaration
\begin{verbatim}
UpperHalf.certifiedUpperHalfLaw : UpperHalf.SharpUpperHalfLaw
\end{verbatim}
combines a universal, arbitrary-rank, normalized-selection upper bound
with a concrete dataset attaining the matching lower bound. Its
definitions include the evaluation risk, support size, selected feature
span, and both minimization requirements in the minimum-norm learner.
The equality for $\Gamma_d$ follows from these two dataset-level
statements and \eqref{eq:gamma}.

\begin{center}
\small
\begin{tabularx}{\linewidth}{@{}lX@{}}
\toprule
Lean module & Mathematical role \\
\midrule
\texttt{NontrivialAnchor} & Direct positive-basis simplex-anchor proof \\
\texttt{MaximalRegression} & Independent-line structure for the complete dataset \\
\texttt{LineCompression} & Legal $2d-t$-row selection and sharp line-risk bound \\
\texttt{AnchorLift}, \texttt{CoordinateAnchor} & Actual weighted regression limit from the projected instance \\
\texttt{IntrinsicReduction} & Deficient-rank coordinates and the minimum-norm rule \\
\texttt{FullRankInduction} & Closed dimension induction, including the zero-gap case \\
\texttt{Sharpness}, \texttt{Certified} & Attaining lower bound and the complete original-data theorem \\
\bottomrule
\end{tabularx}
\end{center}
Run \texttt{setup.ps1} from the formal project, or, with dependencies
installed, run \texttt{lake build} and
\texttt{lake env lean -DwarningAsError=true FullTheoremCheck.lean}.
The complete theorem's axiom audit returns only
\texttt{propext}, \texttt{Classical.choice}, and \texttt{Quot.sound}.
The formalization covers Theorem~\ref{thm:upper-half}; the supplied
symbolic checks for Theorem~\ref{thm:five-six} are documented separately
in Appendix~\ref{app:verification}.

%% file: appendix/landscape.tex
\subsection{A basis facet for one block}
Consider $a+1$ rows spanning $\R^a$, with positive evaluation weights,
positive optimal loss, and every $a$ rows independent. Scale by the square
roots of the weights, whiten the Hessian, and normalize the residual to
obtain $B^\top B=I_a$, $B^\top e=0$, and $\|e\|=1$. Since $B$ has
$a+1$ rows, $BB^\top=I-ee^\top$. Write $b_i^\top$ for its rows.
Choose $i$ maximizing $e_i^2$, so $e_i^2\geq1/(a+1)>0$.
Deleting row $i$ gives normal matrix $I-b_ib_i^\top$ and displacement
of squared norm
\[
 \frac{e_i^2\|b_i\|^2}{(1-\|b_i\|^2)^2}
 =\frac{1-e_i^2}{e_i^2}\leq a.
\]
Here Sherman--Morrison applies because $1-\|b_i\|^2=e_i^2>0$.
Thus a basis facet has risk at most $(a+1)$ times the block optimum.
Its unique interpolant is independent of the positive training weights.
For zero optimal loss, any basis facet is exact.

\subsection{Sharp translated block constructions}
Fix a partition $d=a_1+\cdots+a_k$ and positive masses $\ell_j$ summing
to one. In mutually orthogonal subspaces use the feature rows
\[
 e_{j,1},\ldots,e_{j,a_j},\quad-\sum_{r=1}^{a_j}e_{j,r},
 \qquad p_{j,r}=\frac{\ell_j}{a_j+1}.
\]
Choose $v$ outside every proper span of a subset of these finitely many
rows and set $w_T=Tv$, $y_i=x_i^\top w_T+1$. Then $\wstar=w_T$,
$\LD=1$, and $L_D(w)=1+\|w-w_T\|_H^2$, where
$H=\sum_i p_ix_ix_i^\top\succ0$. All complete blocks are balanced.
Every basis facet in block $j$ interpolates its selected rows and has
block loss $(a_j+1)\ell_j$, hence excess $a_j\ell_j$.

Any full-rank support spends at least $a_j$ rows in each block. At budget
$d+s$ it completes at most $s$ blocks, giving minimal excess
\begin{equation}
 \min_{\substack{J\subseteq[k]\\|J|=s}}\sum_{j\notin J}a_j\ell_j
 \quad(s<k).
 \label{eq:landscape-allocation}
\end{equation}
For a deficient support $S$, the minimum-norm output belongs to its
proper feature span $U_S$. The finite family of such spans satisfies
$\rho=\min_S\dist_H(v,U_S)>0$. Its risk is at least $1+T^2\rho^2$.
A sufficiently large finite $T$ makes this exceed every candidate in
\eqref{eq:landscape-allocation}. Thus that formula is also exact when
the selection may lose rank. Taking
\[
 c=\left(\sum_j a_j^{-1}\right)^{-1},\qquad \ell_j=c/a_j
\]
makes every unfinished block cost $c$, proving the general lower bound
in Proposition~\ref{prop:landscape}. The constructed residual-gradient
atoms form positive circuits, so the examples belong to the stated class.

\subsection{Optimization over block losses}
For a fixed partition define
\[
 \Psi_s(a)=\max_{\ell\in\Delta_k}
       \min_{|J|=s}\sum_{j\notin J}a_j\ell_j.
\]
For $s\geq k$ put $\Psi_s=0$. For $s<k$, the sum of the $k-s$
smallest coordinates of $c_j=a_j\ell_j$ is
$\min_{0\leq\alpha_j\leq1,\,\sum\alpha_j=k-s}\sum\alpha_jc_j$.
Finite-dimensional minimax on these compact polytopes gives
\[
 \Psi_s(a)=\min_{0\leq\alpha_j\leq1,\,\sum\alpha_j=k-s}
             \max_j a_j\alpha_j
 =\max_{\substack{R\subseteq[k]\\|R|>s}}
       \frac{|R|-s}{\sum_{j\in R}a_j^{-1}}.
\]
Indeed, a candidate $t$ is feasible in the minimum exactly when
$\sum_j\min\{1,t/a_j\}\geq k-s$. For each subset $R$ this requires
$|R^c|+t\sum_{j\in R}a_j^{-1}\geq k-s$; conversely, using
$R=\{j:a_j>t\}$ proves sufficiency (the case $|R|\leq s$ is immediate).

For arbitrary blocks in the stated class, complete the $s$ blocks with
largest $a_jL_j^\star$ and use the preceding basis-facet bound on every
other block. The resulting selection has full rank, at most $d+s$ rows,
and excess ratio at most $\Psi_s(a)$. Orthogonality makes these block
solutions compatible with the global minimum-norm learner.
The translated construction attains the same worst-case value in the
supremum: approximate any zero maximizing masses by positive masses.

Finally, an active subset $R$ uses at most $d$ dimensions. Distributing
all remaining dimensions among its blocks decreases their reciprocal
sum and increases its candidate value. Hence maximizing $\Psi_s(a)$
over all partitions of $d$ is exactly $\Phi_{d,s}$. This proves the
structured equality in Proposition~\ref{prop:landscape}.

\subsection{Evaluating the frontier}
For fixed $k$, replacing integers $a\geq b+2$ by $a-1,b+1$ strictly
decreases $1/a+1/b$. Thus the optimal partition is balanced. With
$d=qk+r$, $0\leq r<k$, this gives
\begin{equation}
 \Phi_{d,s}=\max_{s<k\leq d}
   \frac{k-s}{(k-r)/q+r/(q+1)}.
 \label{eq:landscape-balanced}
\end{equation}
For $s=1$ and $k\geq3$, Cauchy--Schwarz bounds the candidate by
$d(k-1)/k^2\leq2d/9$. For $d\geq3$ this is at most
$(d^2-1)/(4d)\leq\lfloor d^2/4\rfloor/d$. The two-block balanced
partition attains $\lfloor d^2/4\rfloor/d$. For $d=2$ only $k=2$
is available. This proves the stated formula for $\Phi_{d,1}$.
At $(d,s)=(5,1)$ the optimal partition is $3+2$, with masses $2/5,3/5$.

\subsection{Sparse weighted orthogonalization as the next target}
\label{app:matrix-target}
Fix full observed rank $d$, write $H=\sum_i p_ix_ix_i^\top$, and set
\[
 B_i=\sqrt{p_i}\,x_i^\top H^{-1/2},\qquad
 e_i=\sqrt{p_i}(y_i-x_i^\top\wstar)/\sqrt{\LD}.
\]
Then $B^\top B=I_d$, $B^\top e=0$, $\|e\|=1$.
For a full-rank training selection $q$, let
$\Lambda=\diag(q_i/p_i)$. Its normalized risk is exactly
\begin{equation}
 \frac{L_D(\widehat w(q))}{\LD}
 =1+\left\|(B^\top\Lambda B)^{-1}B^\top\Lambda e\right\|^2.
 \label{eq:landscape-matrix}
\end{equation}
Conversely, any nonzero nonnegative diagonal $\Lambda$ with positive
definite $B^\top\Lambda B$ defines the same fit by normalizing
$q_i\propto p_i\Lambda_{ii}$. This transformation keeps every row,
including zero residuals.

The approximation form $\mathsf W_\varepsilon(d,s;c)$ asks that, for
every such finite $(B,e)$ and every $\varepsilon>0$, there exist
nonnegative diagonal $\Lambda$ with at most $d+s$ positive entries,
$B^\top\Lambda B\succ0$, and
\[
 \left\|(B^\top\Lambda B)^{-1}B^\top\Lambda e\right\|^2
 \leq c+\varepsilon.
\]
By \eqref{eq:landscape-matrix}, this is equivalent to the universal
full-rank, span-preserving selected-risk infimum bound $1+c$.
For the converse representation of arbitrary $(B,e)$, take
$p_i=1/N$, $x_i=\sqrt N B_i^\top$, and $y_i=\sqrt N e_i$.
The exact-attainment version is the stronger target with
$\varepsilon=0$ required for a single admissible selection.

The general block prediction suggests $c=\Phi_{d,s}$. To cover ambient
dimension $d$ including intrinsic ranks $r<d$, one can seek the same
matrix bound in each intrinsic rank with budget $n=d+s$ and the ambient
constant $\Phi_{d,s}$; ranks with $n\geq2r$ already admit exact
certificates. This states the remaining problem with its full budget
and rank requirements. Theorem~\ref{thm:upper-half} supplies the bound
$c=(2d-n)/d$ throughout $\lceil3d/2\rceil\leq n\leq2d-1$.
At $(5,6)$, the span-preserving assertion of
Theorem~\ref{thm:five-six} supplies the bound $c=6/5$.

%% file: appendix/verification.tex
\subsection{Reproduction and scope}
The anonymous source package includes
\path{verification/verify_exact.py}. With Python~3.12 and SymPy~1.13.2,
run the following from the package root:
\begin{verbatim}
python -m pip install -r verification/requirements.txt
python verification/verify_exact.py
\end{verbatim}
The script uses exact rational arithmetic and symbolic identities. It
writes \path{verification/exact_check_results.json}, including the
hashes of the script and the five-dimensional proof source. It requires
no dataset, network connection after installation, or access to a prior
draft. The reported checks cover algebraic identities, selected symbolic
ranks, a finite rational grid for one boundary-rank diagnostic, all
nonempty supports of size at most six in the seven-point example, and
the explicit path/cycle/star probability constructions. The script is
not a formal proof assistant and does not test all root configurations.

The lower-bound support check exhausts all 126 nonempty supports: 19
have full rank and 107 have deficient rank. On the full-rank supports it
checks the block counts used in the analytic argument; on every deficient
support it computes the minimum evaluation-norm squared distance from the
translation vector to the selected span. The smallest such distance is
five. It also solves the stated attaining full-rank selection and obtains
the ratio $11/5$ exactly. The analytic interpolation argument in
Appendix~\ref{app:five-lower}, rather than a finite sample of training
weights, covers arbitrary positive weights on each full-rank support.

\subsection{Algebra-to-proof mapping}
The diagnostic identifiers are stable local equation identifiers used by
the checker. Their correspondence with the paper is as follows.
\begin{center}
\small
\begin{tabularx}{\linewidth}{@{}lX@{}}
\toprule
Checker identifiers & Mathematical object in this paper \\
\midrule
9.1 & Same-type crossing relation, \eqref{eq:five-AA} \\
9.2 & Same-index mixed relation, \eqref{eq:five-ABsame} \\
9.3 & Cross-index positive representation, \eqref{eq:five-ABcross} \\
9.4 & Type-C/type-A cross relation, \eqref{eq:five-CAcross} \\
9.5 & Shared-index circuit relation, \eqref{eq:five-CAsame} \\
12.2--12.8 & Block-restriction representations following
\eqref{eq:five-blockrestriction} \\
13.1--13.3 & Remaining type-C relations,
\eqref{eq:five-C01}--\eqref{eq:five-C22} \\
Support checks & Seven-point construction,
Appendix~\ref{app:five-lower} \\
Graph checks & Explicit probabilities in Lemma~\ref{lem:five-graph} \\
\bottomrule
\end{tabularx}
\end{center}
Generic symbolic rank calculations do not by themselves establish rank
at every positive specialization. The paper gives the structural rank
and positivity arguments; the finite boundary-rank grid is explicitly
diagnostic. Likewise, the graph checks evaluate the listed probabilities
but do not replace the circuit-plane intersection lemmas or the transfer
from graph probabilities to actual regression risk.

\subsection{Dependency index}
For the upper bound, Appendix~\ref{app:five-interfaces} establishes the
training interface and comparison moment; Appendix~\ref{app:five-polar}
proves the active-space and polar-slice restrictions;
Appendix~\ref{app:five-shared} closes the shared-root rank-three case;
Appendix~\ref{app:five-lowrank} treats the rank-two plane graph;
Appendix~\ref{app:five-remaining} closes the other rank-three case;
and Appendix~\ref{app:five-transfer} handles all regression degeneracies.
The only lower-dimensional interior input is the five-point upper bound
in four dimensions from \citet[Theorem~47]{zhang2026exact}.
The dimension-three boundary input is proved in Appendix~\ref{app:proofs}.
The exact lower bound is in Appendix~\ref{app:five-lower}.